\documentclass{article}

\usepackage{microtype}
\usepackage{graphicx}
\usepackage{booktabs} % for professional tables
\usepackage{caption}
\usepackage{subcaption}
\usepackage{makecell}
\usepackage{multirow}

\usepackage{amsmath}
\usepackage{amssymb}
\usepackage{amsfonts}
\usepackage{amsthm}
\usepackage{mathtools}

\usepackage{color}
\usepackage{xcolor,colortbl}

\usepackage{tikz}
\usetikzlibrary{arrows}
\usetikzlibrary{automata}
\usetikzlibrary{matrix}
\usetikzlibrary{calc}
\usetikzlibrary{positioning}
\usetikzlibrary{graphs}
\usetikzlibrary{shapes.geometric}
\usetikzlibrary{backgrounds}

\usepackage{enumitem}

\usepackage{alphalph}
\usepackage{pgffor}

\usepackage[linesnumbered,ruled,noend]{algorithm2e}

\newtheoremstyle{thm}
  {2pt} % Space above
  {2pt} % Space below
  {\itshape} % Body font
  {} % Indent amount
  {\bfseries} % Theorem head font
  {.} % Punctuation after theorem head
  {.5em} % Space after theorem head
  {} % Theorem head spec (can be left empty, meaning `normal')
\theoremstyle{thm}

\usepackage{thmtools} 
\usepackage{thm-restate}

\newtheorem{theorem}{Theorem}
\newtheorem{lemma}{Lemma}

\newtheorem{definition}{Definition}

\usepackage[colorlinks]{hyperref}
\usepackage{textcomp}

\usepackage{etoolbox}
\preto\equation{\par\nobreak\small\noindent}
\preto\align{\par\nobreak\small\noindent}
\expandafter\preto\csname equation*\endcsname{\par\nobreak\small\noindent}
\expandafter\preto\csname align*\endcsname{\par\nobreak\small\noindent}

\definecolor{matlab-blue}{rgb}{0,0.4470,0.7410}
\definecolor{matlab-orange}{rgb}{0.8500,0.3250,0.0980}
\definecolor{matlab-yellow}{rgb}{0.9290,0.6940,0.1250}
\definecolor{matlab-green}{rgb}{0.4660,0.6740,0.1880}
\definecolor{matlab-red}{rgb}{0.6350,0.0780,0.1840}
\definecolor{matlab-purple}{rgb}{0.4901,0.1803,0.5529}
\definecolor{matlab-light-blue}{rgb}{0.298,0.741,0.929}
\definecolor{ourmethod}{gray}{0.93}

\newcommand{\define}{\triangleq}
\newcommand{\btheta}{\boldsymbol{\theta}}
\newcommand{\bs}{\mathbf{s}}
\newcommand{\ba}{\mathbf{a}}
\newcommand{\bw}{\mathbf{w}}
\newcommand{\D}{\mathcal{D}}
\newcommand{\A}{\mathcal{A}}
\renewcommand{\S}{\mathcal{S}}
\newcommand{\B}{\mathcal{B}}

\renewcommand{\L}{\mathcal{L}}
\newcommand{\C}{\mathcal{C}}

\DeclareMathOperator*{\E}{\mathbb{E}}
\DeclareMathOperator*{\argmax}{arg\,max}
\DeclareMathOperator*{\argmin}{arg\,min}
\newcommand{\underdescribe}[3][0pt]{\hspace*{.12em}\underbracket[0.5pt][1pt]{#2\hspace*{#1}}_{#3}}

\usepackage[accepted]{icml2021}

\icmltitlerunning{$\Psi \Phi$-Learning: RL with Demonstrations using Successor Features and Inverse TD Learning}

\newcommand{\bftab}{\fontseries{b}\selectfont}

\begin{document}

\twocolumn[
\icmltitle{PsiPhi-Learning: Reinforcement Learning with Demonstrations using Successor Features and Inverse Temporal Difference Learning}

\icmlsetsymbol{equal}{*}

\begin{icmlauthorlist}
\icmlauthor{Angelos Filos}{ox}
\icmlauthor{Clare Lyle}{ox}
\icmlauthor{Yarin Gal}{ox}
\icmlauthor{Sergey Levine}{ucb}
\icmlauthor{Natasha Jaques}{equal,ucb,brain}
\icmlauthor{Gregory Farquhar}{equal,dm}
\end{icmlauthorlist}

\icmlaffiliation{ox}{University of Oxford}
\icmlaffiliation{ucb}{University of California, Berkeley}
\icmlaffiliation{brain}{Google Research, Brain team}
\icmlaffiliation{dm}{DeepMind}

\icmlcorrespondingauthor{Angelos Filos}{angelos.filos@cs.ox.ac.uk}

\icmlkeywords{
  Machine Learning,
  Reinforcement Learning,
  Imitation Learning,
  Successor Features,
  Transfer Learning,
  Multi-Agent,
  Autonomous Driving,
}

\vskip 0.3in
]

\printAffiliationsAndNotice{\icmlEqualContribution}

\begin{abstract}
We study reinforcement learning (RL) with no-reward demonstrations, a setting in which an RL agent has access to additional data from the interaction of other agents with the same environment.
However, it has no access to the rewards or goals of these agents, and their objectives and levels of expertise may vary widely.
These assumptions are common in multi-agent settings, such as autonomous driving.
To effectively use this data, we turn to the framework of successor features.
This allows us to disentangle shared features and dynamics of the environment from agent-specific rewards and policies.
We propose a multi-task inverse reinforcement learning (IRL) algorithm, called \emph{inverse temporal difference learning} (ITD), that learns shared state features, alongside per-agent successor features and preference vectors, purely from demonstrations without reward labels.
We further show how to seamlessly integrate ITD with learning from online environment interactions, arriving at a novel algorithm for reinforcement learning with demonstrations, called $\Psi \Phi$-learning (pronounced `Sci-Fi').
We provide empirical evidence for the effectiveness of $\Psi \Phi$-learning as a method for improving RL, IRL, imitation, and few-shot transfer, and derive worst-case bounds for its performance in zero-shot transfer to new tasks.
% \vspace{-0.5em}
\end{abstract}

% \vspace{-2.35em}
\section{Introduction}
\label{sec:introduction}

If artificial agents are to be effective in the real world, they will need to thrive in environments populated by other agents.
Agents are typically goal-directed, sometimes by definition \citep{franklin1996agent}. While their goals can be different, they often depend on shared salient features of the environment, and may be able to interact with and affect the environment in similar ways.
Humans and other animals make ready use of these similarities to other agents while learning \citep{henrich2017secret,laland2018darwin}.
We can observe the goal-directed behaviours of other humans, and combine these observations with our own experiences, to quickly learn how to achieve our own goals.
If reinforcement learning~\citep[RL,][]{sutton2018reinforcement} agents could similarly interpret the behaviour of others, they could learn more efficiently, relying less on solitary trial and error.

\begin{figure}[t]
  \centering
  \includegraphics[width=\linewidth]{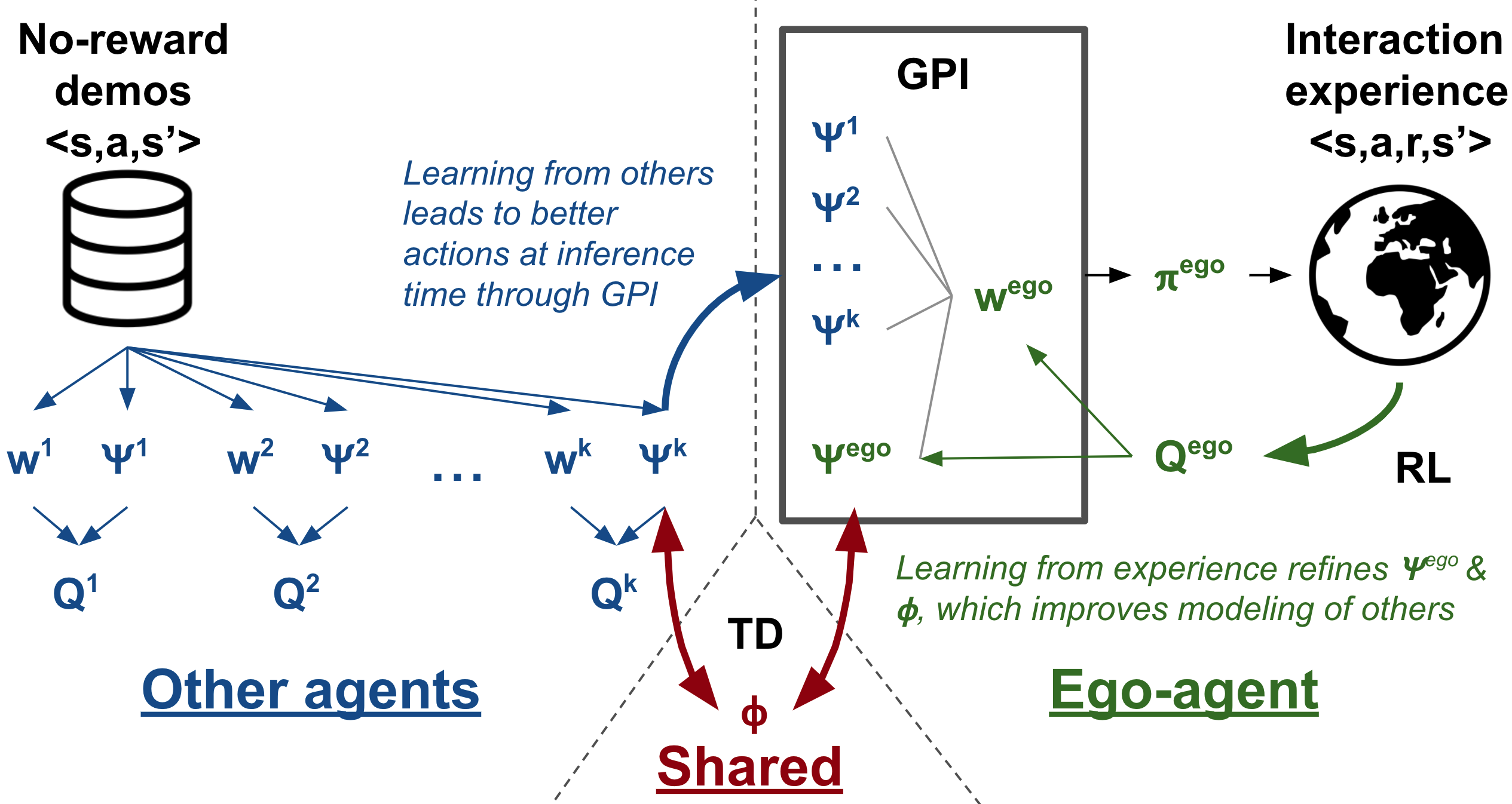}
  \caption{
    \textbf{The $\Psi \Phi$-learning algorithm for RL with no-reward demonstrations.}
    Demonstrations $\D$ contain data from {\color{matlab-blue}other agents} for \emph{unknown} tasks. We employ \emph{inverse temporal difference learning} (ITD, cf. Section~\ref{subsec:inverse-temporal-difference-learning}) to recover other agents' successor features (SFs) $\Psi^k$ and preferences $\bw^k$.
    The {\color{matlab-green}ego-agent} combines the estimated SFs of others along with its own preferences $\bw^{\text{ego}}$ and SFs $\Psi^{\text{ego}}$ with generalised policy improvement (GPI, cf. Section~\ref{subsec:successor-features-and-cumulants}), generating experience.
    Both the demonstrations and the ego-experience are used to learn the {\color{matlab-red}shared cumulants} $\Phi$.
}
\label{fig:implementation}
% \vspace{-2em}
\end{figure}

To this end, we formalise and address a problem setting in which an agent (the `ego-agent’) is given access to observations and actions drawn from the experiences of other goal-directed agents interacting with the same environment, but pursuing distinct goals.
These observed trajectories are unlabelled in the sense that they lack the goals or rewards of the other agents.
This type of data is readily available in many real-world settings, either from (i) observing other agents acting simultaneously with the ego-agent in the same (multi-agent) environment, or (ii) multi-task demonstrations collected independently from the ego-agent's experiences.
Consider autonomous driving as a motivating example: the car can observe the decisions of many nearby human drivers with various preferences and destinations, or may have access to a large offline dataset of such demonstrations.
Because the other agents are pursuing their own varied goals, it can be difficult to directly use this information with conventional imitation learning methods \citep{widrow1964pattern} or inverse RL (IRL) \citep{ng2000algorithms,ziebart2008maximum}.

While the ego-agent should not copy other agents directly, it is likely that the behaviour of all agents depends on shared features of the environment.
To disentangle such shared features from agent-specific goals, we turn to the framework of successor features \citep{dayan1993improving,kulkarni2016deep,barreto2017successor}.
Successor features are a representation that captures the sum of state features an agent's policy will reach in the future.
An agent's goal is represented separately as a preference vector.

In this paper, we demonstrate how a reinforcement learner can benefit from multi-task demonstrations using the framework of successor features.
The key contributions are:
\begin{enumerate}
% \vspace{-1em}
  \item \textbf{Offline multi-task IRL:}
    We propose an inverse RL algorithm, called \emph{inverse temporal difference} (ITD) \emph{learning}.
    Using only demonstrations, we learn shared state features, alongside per-agent successor features and inferred preferences.
    The reward functions can be trivially computed from these learned quantities.
    We show empirically that ITD achieves superior or comparable performance to prior methods.

  \item  \textbf{RL with no-reward demonstrations:}
    By combining ITD with learning from environment interactions, we arrive at a novel algorithm for RL with unlabelled demonstrations, called $\Psi \Phi$-learning (pronounced `Sci-Fi').
    $\Psi \Phi$-learning is compatible with sub-optimal demonstrations.
    It treats the demonstrated trajectories as being soft-optimal under \emph{some} task and employs ITD to recover successor features for the demonstrators' policies.
    $\Psi \Phi$-learning inherits the unbiased, asymptotic performance of RL methods while leveraging the provided demonstrations with ITD.
    When the goals of any of the demonstrators are even partially aligned with the $\Psi \Phi$-learner, this enables much faster learning than solitary RL.
    Otherwise, when the demonstrations are not useful or even misleading, it gracefully falls back to standard RL, unlike na{\"i}ve behaviour cloning or IRL.

  \item   \textbf{Few-shot adaptation with task inference:}
    Taking full-advantage of the successor features framework, our $\Psi \Phi$-learner can even adapt zero-shot to new goals it has never seen or experienced during training, but which are partially aligned with the demonstrated goals.
    This is possible due to the disentanglement of representations into task-specific features (i.e., preferences) and shared state features.
    We can efficiently update the task-specific preferences and rely on generalised policy improvement for safe policy updates.
    We derive worst-case bounds for the performance of $\Psi \Phi$-learning in zero-shot transfer to new tasks.
\end{enumerate}

We evaluate $\Psi \Phi$-learning in a set of grid-world environments, a traffic-flow simulator~\citep{highway-env}, and a task from the ProcGen suite~\cite{cobbe2020leveraging}, observing advantages over vanilla RL, imitation learning \cite{reddy2019sqil,ho2016generative}, and auxiliary-task baselines \cite{hernandez2017survey}.
Thanks to the shared state features between the ITD and RL components, we find empirically that the $\Psi \Phi$-learner not only improves its ego-learning with demonstrations, but also enhances its ability to model others agents using its own experience.

% \vspace{-1em}
\section{Background and Problem Setting}
\label{sec:background-and-problem-setting}

We consider a world that can be represented as an infinite horizon controlled Markov process (CMP) given by the tuple: $\C \define  \langle \mathcal{S}, \mathcal{A}, P, \gamma \rangle$.
$\mathcal{S}$ and $\mathcal{A}$ represent the continuous state and discrete action spaces, respectively, $\bs' \sim P(\cdot \vert \bs, \ba)$ describes the transition dynamics and $\gamma$ is the discount factor.
A \textbf{task} is formulated as a Markov decision process~\citep[MDP,][]{puterman2014markov}, characterised by a reward function, $R: \mathcal{S} \times \mathcal{A} \rightarrow \mathbb{R}$, i.e., $M \define  \langle \C, R \rangle$.

The goal of an agent is to find a policy which maps from states to a probability distribution over actions, $\pi: \mathcal{S} \rightarrow \Delta(\mathcal{A})$, maximising the expected discounted sum of rewards $G^{R} \define \sum_{t=0}^{\infty} \gamma^{t} R(\bs_t, \ba_t)$.
The action-value function of the policy $\pi$ is given by $Q^{\pi, R}(\bs, \ba) \define \E^{\C, \pi} \left[ \left. G^{R} \right| \bs_0 = \bs, \ba_0 = \ba \right]$, where $\E^{\C, \pi} \left[ \cdot \right]$ denotes expected value when following policy $\pi$ in environment $\C$.

\subsection{RL with No-Reward Demonstrations}
\label{subsec:rl-with-no-reward-demonstrations}

We are interested in settings in which, in addition to an environment $\C$, the agent has also access to \textbf{demonstrations without rewards}, i.e., behavioural data of mixed and unknown quality.
The demonstrations are generated by other agents, whose goals and levels of expertise are unknown, and who have no incentive to educate the controlled agent.
We will refer to the controlled agent, i.e., reinforcement learner, as the `ego-agent' and to the agents that generated the demonstrations as `other-agents'.
We denote the demonstrations with $\D = \{ \tau_1, \tau_2, \ldots, \tau_N \}$, where the trajectory $\tau \define \left( \bs_0, \ba_0, \ldots, \bs_T, \ba_T; k \right)$ is generated by the $k$-th agent.
Note that each trajectory does include an identifier of the agent that generated it.
The ego-agent also gathers its own experience by interacting with the environment, collecting data $\B = \{ \left( \bs, \ba, \bs', r^{\text{ego}} \right) \}$.
Due to the lack of reward annotations in $\D$, and the fact that the data may be irrelevant to the ego-agent's task, it is not trivial to combine demonstrations from $\D$ with the ego-agent's experience $\B$.

\subsection{Successor Features and Cumulants}
\label{subsec:successor-features-and-cumulants}

To make use of the demonstrations $\D$, we wish to capture the notion that while the agents' rewards may differ, they share the same environment.
To do so we turn to the framework of successor features (SFs) ~\citep{barreto2017successor}, in which rewards are decomposed into cumulants and preferences:
\begin{definition}[Cumulants and Preferences]
  The (one-step) rewards are decomposed into task-agnostic \emph{cumulants} $\Phi(\bs, \ba) \in \mathbb{R}^{d}$, and task-specific \emph{preferences} $\bw \in \mathbb{R}^{d}$:
  \begin{equation}
    R^{\bw}(\bs, \ba) \define \Phi(\bs, \ba)^{\top} \bw \,.
  \label{eq:cumulants-times-preferences}
  \end{equation}
\end{definition}
%
% \vspace{-0.5em}
%
The preferences $\bw$ are a representation of a possible goal in the world $\C$, in the sense that each $\bw$ gives rise to a task $M^{\bw} = \langle \C, R^{\bw} \rangle$.
We use `task', `goal', and `preferences' interchangeably when context makes it clear whether we are referring to $\bw$ itself, or the corresponding $M^{\bw}$ or $R^{\bw}$.
% Consequently, we use `task $\bw$' to refer to the corresponding $M^{\bw}$.
The action-value function for a policy $\pi$ in $M^{\bw}$ is then a function of the preferences $\bw$ and the $\pi$'s successor features.
% \vspace{-1em}
%
\begin{definition}[Successor Features]
  For a given discount factor $\gamma \in [0, 1)$, policy $\pi$ and cumulants $\Phi(\bs, \ba) \in \mathbb{R}^{d}$, the successor features (SFs) for a state $\bs$ and action $\ba$ are:
  \begin{equation}
    \Psi^{\pi}(\bs, \ba) \define \mathbb{E}^{\C, \pi} \left[ \left. \sum_{t=0}^{\infty} \gamma^{t} \Phi(\bs_t, \ba_t)  \right| \bs_0 = \bs, \ba_0 = \ba \right] \,.
    \label{eq:sf-definition}
  \end{equation}
\end{definition}
The $i$-th component of $\Psi^{\pi}(\bs, \ba)$ gives the expected discounted sum of $\Phi(\bs, \ba)$'s $i$-th component, when starting from state $\bs$, taking action $\ba$ and then following policy $\pi$.
Intuitively, cumulants $\Phi$ can be seen as a vector-valued reward function and SFs $\Psi^{\pi}$ the corresponding vector-valued state-action value function for policy $\pi$.

An action-value function is then given by the dot product of the preferences $\bw$ and $\pi$'s SFs:
\begin{equation}
Q^{\pi, \bw}(\bs, \ba) = \Psi^{\pi}(\bs, \ba)^{\top} \bw \,.
\label{eq:Q-Psi-times-w}
\end{equation}
\begin{proof}
% \vspace{-1em}
See~\citep{barreto2017successor}.
% \vspace{-1em}
\end{proof}
Note that if we have $\Psi^{\pi}$, the value of $\pi$ for a new preference $\bw'$ can be easily computed.
This property allows the successor features of a set of policies to be repurposed for accelerating policy updates, as follows.
\begin{definition}[Generalised Policy Improvement]
  Given a set of policies $\Pi = \{ \pi_{1}, \ldots, \pi_{K} \}$ and a task with reward function $R$, generalised policy improvement (GPI) is the definition of a policy $\pi'$ s.t.
  \begin{equation}
    Q^{\pi', R}(\bs, \ba) \geq \sup_{\pi \in \Pi} Q^{\pi, R}(\bs, \ba) \,, \forall \bs \in \S, \ba \in \A \,.
  \label{eq:gpi}
  \end{equation}
\end{definition}
Provided the SFs of a set of policies, i.e., $\{\Psi^{\pi^{k}}\}_{k=1}^{K}$, we can apply GPI to derive a new policy $\pi'$ whose performance on a task $\bw$ is no worse that the performance of any of $\pi \in \Pi$ on the same task, given by
\begin{equation}
  \pi'(\bs) = \argmax_{a} \max_{\pi \in \Pi} \Psi^{\pi}(\bs, a)^{\top} \bw \,.
\label{eq:gpi-with-sf}
\end{equation}
Eqn.~(\ref{eq:gpi-with-sf}) suggests that if we could estimate the SFs of other agents, we could utilise them for improving the ego-agent's policy with GPI.
However, to do so with conventional methods we would require access to their rewards, cumulants and/or preferences.
In our setting, we can only observe their sequence of states and actions (Section~\ref{subsec:rl-with-no-reward-demonstrations}).
Next, we introduce our method that only requires no-reward demonstrations to estimate SFs and can be integrated seamlessly with GPI for accelerating reinforcement learning.
% \vspace{-1em}

\section{Accelerating RL with Demonstrations}
\label{sec:method}

We now present a novel method, $\Psi \Phi$-learning, that leverages reward-free demonstrations to accelerate RL, shown in Figure~\ref{fig:implementation}.
Our approach consists of two components: (i) a novel inverse reinforcement learning algorithm, called \emph{inverse temporal difference} (ITD) \emph{learning}, for learning cumulants, per-agent successor features and corresponding agent preferences from demonstration without reward labels, and (ii) a novel RL %with reward-free demonstrations 
algorithm that combines ITD with generalised policy improvement (GPI, Section~\ref{subsec:successor-features-and-cumulants}).

\subsection{Inverse Temporal Difference Learning}
\label{subsec:inverse-temporal-difference-learning}

Given demonstrations without rewards, $\D$, we model the agents that generated the data (i.e., blue nodes in Figure~\ref{fig:implementation}) as soft-optimal for an \emph{unknown} task.
In particular, the $k$-th agent's policy is soft-optimal under task $\bw^{k}$ and is given by
\begin{equation}
  \pi^{k}(\ba | \bs) = \frac{\exp(\Psi^{\pi^{k}}(\bs, \ba)^{\top} \bw^{k})}{\sum_{a} \exp(\Psi^{\pi^{k}}(\bs, a)^{\top} \bw^{k})} \,, \forall \bs \in \S, \ba \in \A \,.
\label{eq:pi-k}
\end{equation}
We choose to represent the action-value functions of the other agents with their SFs and preferences to enable GPI, and to expose task- and policy-agnostic structure in the form of shared cumulants $\Phi$.
The $k$-th agent's successor features are temporally consistent with these cumulants $\Phi$
\begin{equation}
  \Psi^{\pi^{k}}(\bs, \ba) = \Phi(\bs, \ba) + \gamma\mathbb{E}^{\C, \pi^{k}} \left[ \Psi^{\pi^{k}}(\bs', \pi^{k}(\bs))  \right] \,.
\label{eq:psi-k}
\end{equation}
To learn these quantities from $K$ demonstrators, we parameterise the SFs with $\btheta^{\Psi^{k}}$, preferences with $\bw^{k}$, and shared cumulants with $\btheta_{\Phi}$.
A schematic of the model architecture and further details are provided in Appendix~\ref{app:implementation-details}.
The parameters are learned by minimising a behavioural cloning and SFs TD loss based on equations \eqref{eq:pi-k} and \eqref{eq:psi-k}.

\textbf{Behavioural cloning loss.}
~Given demonstrations generated only by the $k$-th agent, i.e., $\D^{k} \subset \D$, we train its successor features $\btheta_{\Psi^{k}}$ and the preferences $\bw^{k}$ by minimising the negative log-likelihood of the demonstrations
\begin{equation}
\begin{tiny}
  \! \! \! \! \! \L_{\text{BC-}Q}(\btheta_{\Psi^{k}}, \bw^{k}) \define - \! \! \! \! \! \E_{\tau \sim \D^{k}} \! \! \! \log \frac{\exp(\Psi(\bs_{t}, \ba_{t}; \btheta_{\Psi^{k}})^{\top} \bw^{k})}{\sum_{a} \exp(\Psi(\bs_t, a; \btheta_{\Psi^{k}})^{\top} \bw^{k})} \,.
\label{eq:bcq-loss}
\end{tiny}
\end{equation}
Importantly, Eqn.~(\ref{eq:bcq-loss}) reflects the fact that along a trajectory $\tau$, the successor features are a function of the state and action at each time-step, $\bs_t$ and $\ba_t$, while the preferences $\bw^{k}$ are learnable but consistent across time and trajectories.
The direction of the preference $\bw^k$ indicates the goal of the agent $k$ by showing how relatively rewarding it finds the different dimensions of the cumulant features $\Phi$.
The learned magnitude of $\bw^k$ can further capture how greedily the agent $k$ pursues this goal.

A sparsity prior, i.e., $\L_{1}$ loss, on preferences $\bw^{k}$ is also used to promote disentangled cumulant dimensions (see Figure~\ref{fig:cumulants}).
We found the $\L_{1}$ loss made the algorithm more robust to the choice of dimension of $\Phi$ (see Figure \ref{fig:sensitivity-to-size-of-cumulants} in the Appendix), but did not substantially affect overall performance.

\textbf{Inverse temporal difference loss.}
~The cumulant parameters $\btheta_{\Phi}$ are trained to be TD-consistent with all agents' successor features.
This procedure inverts\footnote{Hence the name \emph{inverse TD learning}.} the standard TD-learning framework for SFs~\citep{dayan1993improving,barreto2017successor} where they are trained to be consistent with a fixed cumulant $\Phi$.
Instead, we first train the SFs and preference vectors to `explain' the other agents' behaviour with the behavioural cloning loss Eqn.~(\ref{eq:bcq-loss}), and then train $\btheta_{\Phi}$, $\btheta_{\Psi^{k}}$ to be (self-)consistent with these SFs by minimising
\begin{align}
  \L_{\text{ITD}}(\btheta_{\Phi}, &\btheta_{\Psi^{k}}) \define \E_{(\bs_{t}, \ba_{t}, \bs_{t+1}, \ba_{t+1}, k) \sim \D} \| \Psi(\bs_{t}, \ba_{t}; \btheta_{\Psi^{k}}) \hookleftarrow \nonumber \\ & - \Phi(\bs_{t}, \ba_{t}; \btheta_{\Phi}) - \underdescribe{\gamma \Psi(\bs_{t+1}, \ba_{t+1}; \tilde{\btheta}_{\Psi^{k}})}{\texttt{stop\_gradient}}  \| \,.
\label{eq:itd-loss}
\end{align}
In practice, our ITD-learning algorithm alternates between minimising $\L_{\text{BC-}Q}$ for training only the successor features and preferences, and $\L_{\text{ITD}}$ for training both the shared cumulants and successor features, provided only with no-reward demonstrations, as illustrated in blue in Figure~\ref{fig:implementation}.

From the definition of cumulants and preferences, we can recover the $k$-th demonstrator's reward, by applying Eqn.~(\ref{eq:cumulants-times-preferences}), i.e., $R^{k}(\bs, \ba) \approx \Phi(\bs, \ba; \btheta_{\Phi})^{\top} \bw^{k}$.
Our ITD algorithm returns both Q-functions that can be used for imitating a demonstrator and an explicit reward function for each agent, requiring only access to demonstrations without any online interaction with the simulator.
Hence ITD-learning is an offline multi-task IRL algorithm.
ITD is summarised in Algorithm~\ref{algo:itd}.
%
% Also, see Appendix~\ref{app:proofs} for a formal proof why the alternating optimisation of $\L_{\text{BC-}Q}$ and $\L_{\text{ITD}}$ leads to the recovery of \emph{valid} cumulants that explain the observed behaviour.
\begin{theorem}[Validity of the ITD Minimiser]
% \vspace{-1em}
% The minimisation of $\L_{\text{BC-}Q}$ followed by $\L_{\text{ITD}}$ leads to the recovery of valid cumulants that explain the observed behaviour
The minimisers of $\L_{\text{BC-}Q}$ and $\L_{\text{ITD}}$ are potentially-shaped cumulants that explain the observed reward-free demonstrations.
\label{thm:itd}
\end{theorem}
\begin{proof}
% \vspace{-1.5em}
See Appendix~\ref{app:proofs}.
\end{proof}

\textbf{Single-task setting.}
~To gain more intuition about the ITD algorithm, consider the simpler case of performing IRL with demonstrations from a single policy.
This obviates the need for a representation of preferences, so we can use $\bw=1$.
In this case $\Psi$ is the action-value function $Q$ and $\Phi$ is simply the reward $R$.
Minimising \eqref{eq:bcq-loss} reduces to finding a $Q$-function whose softmax gives the observed policy, and minimising \eqref{eq:itd-loss} finds a scalar reward that explains the $Q$-function.
Our more general formulation, with cumulants $\Phi$ in place of a scalar reward, allows us to perform ITD-learning on demonstrations from many policies, and to efficiently transfer to new tasks, as we show next.

\subsection{$\Psi \Phi$-Learning with No-Reward Demonstrations}
\label{subsec:psi-phi-learning-with-no-reward-demonstrations}

Now we present our main contribution, $\Psi \Phi$-learning,
which combines our ITD inverse RL algorithm with RL and GPI, using no-reward demonstrations from other agents to accelerating the ego-agent's learning.
$\Psi \Phi$-learning is depicted in Figure~\ref{fig:implementation} and summarised in Algorithm~\ref{algo:itd}.

$\Psi \Phi$-learning is an off-policy algorithm based on Q-learning~\citep{watkins1992q,mnih2013playing}. The action-value function is represented with successor features, $\Psi^{\text{ego}}$, and preferences, $\bw^{\text{ego}}$, as in Eqn.~(\ref{eq:Q-Psi-times-w}).
The ego-agent interacts with the environment, storing its experience in a replay buffer, $\B \leftarrow \B \cup \{ \left( \bs, \ba, \bs', r^{\text{ego}} \right) \}$.
The $\Psi \Phi$-learner also has estimates for the cumulants, per-agent SFs and preferences obtained with ITD from the demonstrations $\D$.

\textbf{Reward loss.}
~The ego-rewards, $r^{\text{ego}}$, are used to ground the cumulants $\Phi$ and the preferences $\bw^{\text{ego}}$, via the loss
\begin{align}
  \L_{\text{R}}(\btheta_{\Phi}, \bw^{\text{ego}}) \define \E_{(\bs, \ba, r^{\text{ego}}) \sim \B} &\| \Phi(\bs, \ba; \btheta_{\Phi})^{\top} \bw^{\text{ego}} - r^{\text{ego}} \| \,.
\label{eq:r-loss}
\end{align}
Importantly, we share the \emph{same} cumulants between the ITD-learning from other agents and the ego-learning, so that they span the joint space of reward functions.
This can be also seen as a representation learning method, where by enforcing all agents, including the ego-agent, to share the same $\Phi$, we transfer information about salient features of the environment from learning about one agent to benefit learning about all agents.

\textbf{Temporal difference learning.}
~The ego-agent's successor features are learned using two losses. First, we train the SFs to fit the Q-values using the Bellman error
\begin{align}
  \mathcal{L}_{Q}(\btheta_{\Psi^{\text{ego}}}) \define & \E_{(\bs, \ba, \bs', r^{\text{ego}}) \sim \B} \| \Psi(\bs, \ba; \btheta_{\Psi^{\text{ego}}})^{\top} \bw^{\text{ego}} \hookleftarrow \nonumber \\ & - r^{\text{ego}} 
  -\underdescribe{ \gamma \max_{a'} \Psi(\bs', \ba'; \tilde{\btheta}_{\Psi^{\text{ego}}})^\top \bw^{\text{ego}}}{\texttt{stop-gradient}} \| \,.
\label{eq:q-loss}
\end{align}

\vspace{4em}
We additionally train the successor features to be self-consistent (i.e. to satisfy Equation~\ref{eq:sf-definition}) using a TD loss $\mathcal{L}_{\text{TD-}\Psi}$. 
\begin{align}
      \mathcal{L}_{\text{TD-}\Psi}(\btheta_{\Psi^{\text{ego}}}) \define & \E_{(\bs, \ba, \bs', \ba') \sim \B} \| \Psi(\bs, \ba; \btheta_{\Psi^{\text{ego}}}) \hookleftarrow \nonumber \\ & - \Phi(\bs, \ba; \tilde{\btheta}_{\Phi}) 
  - \gamma \underdescribe{\Psi(\bs', \ba'; \btheta_{\Psi^{\text{ego}}})}{\texttt{stop-gradient}}\| \,.
\label{eq:psi-loss}
\end{align}

\textbf{GPI behavioural policy.}
~Provided SFs estimates for the other agents, $\{\btheta_{\Psi^{k}}\}_{k=1}^{K}$, and the ego-agent $\btheta_{\Psi^{\text{ego}}}$, and inferred ego preferences, $\bw^{\text{ego}}$, we adopt an action selection mechanism according to the GPI rule in Eqn.~(\ref{eq:gpi-with-sf})
\begin{align}
% \vspace{-0.5em}
  \pi^{\text{ego}}(\bs) = \argmax_{a} \max_{\btheta_{\Psi}} \Psi(\bs, a; \btheta_{\Psi})^{\top} \bw^{\text{ego}} \,.
\label{eq:pi-ego}
\end{align}
The GPI step lets the agent estimate the value of the demonstration policies on its current task, and then copy the policy that it predicts will be most useful.
We combat model overestimation by acting pessimistically with regard to an ensemble of two successor features approximators.
If the agent's estimated values are accurate and the demonstration policies are useful for the ego task, the GPI policy can obtain good performance faster than policy iteration with only the ego value function.
The next section quantifies this claim.

\begin{figure*}
  \centering
  \begin{subfigure}[b]{0.26\linewidth}
    \centering
    \includegraphics[width=\linewidth]{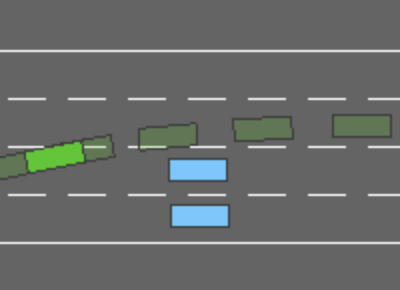}
    \caption{Highway}
  \end{subfigure}
  ~
  \begin{subfigure}[b]{0.238\linewidth}
    \centering
    \includegraphics[width=\linewidth]{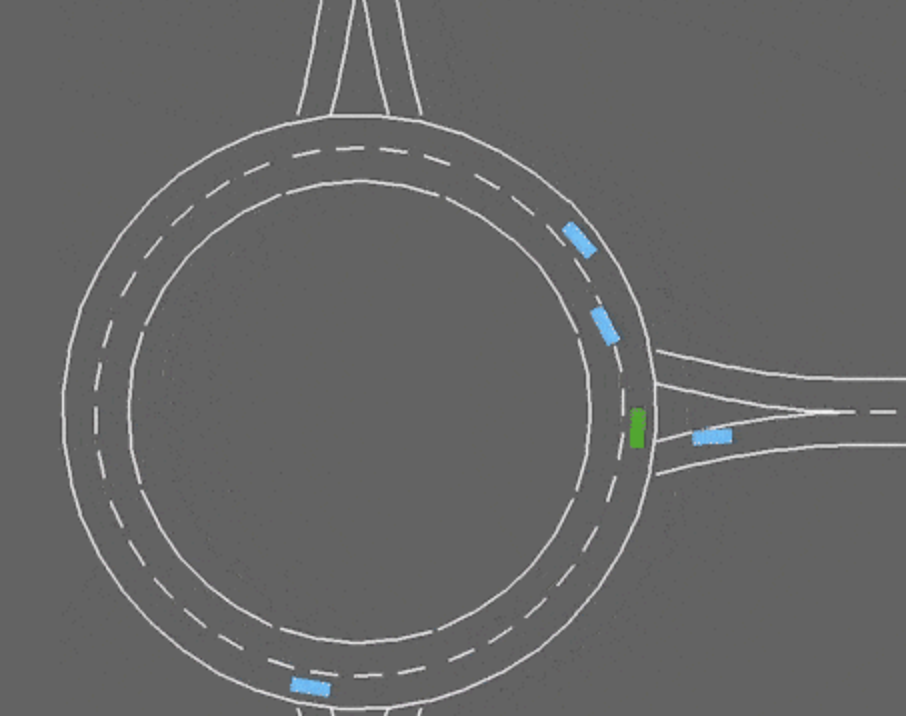}
    \caption{Roundabout}
    \label{fig:envs_roundabout}
  \end{subfigure}
    ~
    \begin{subfigure}[b]{0.19\linewidth}
    \centering
    \includegraphics[width=\linewidth]{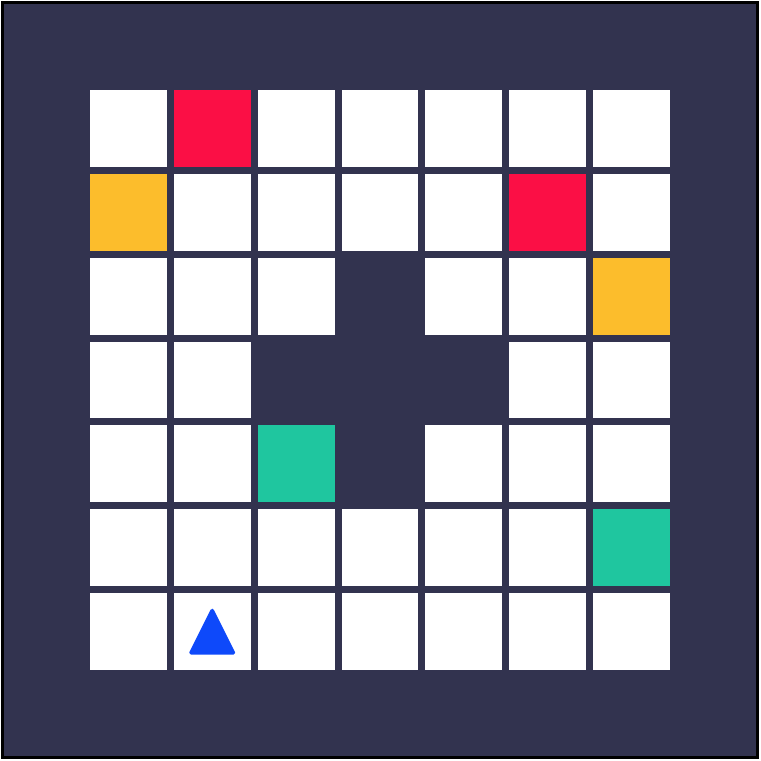}
    \caption{CoinGrid}
  \end{subfigure}
  ~
  \begin{subfigure}[b]{0.189\linewidth}
    \centering
    \includegraphics[width=\linewidth]{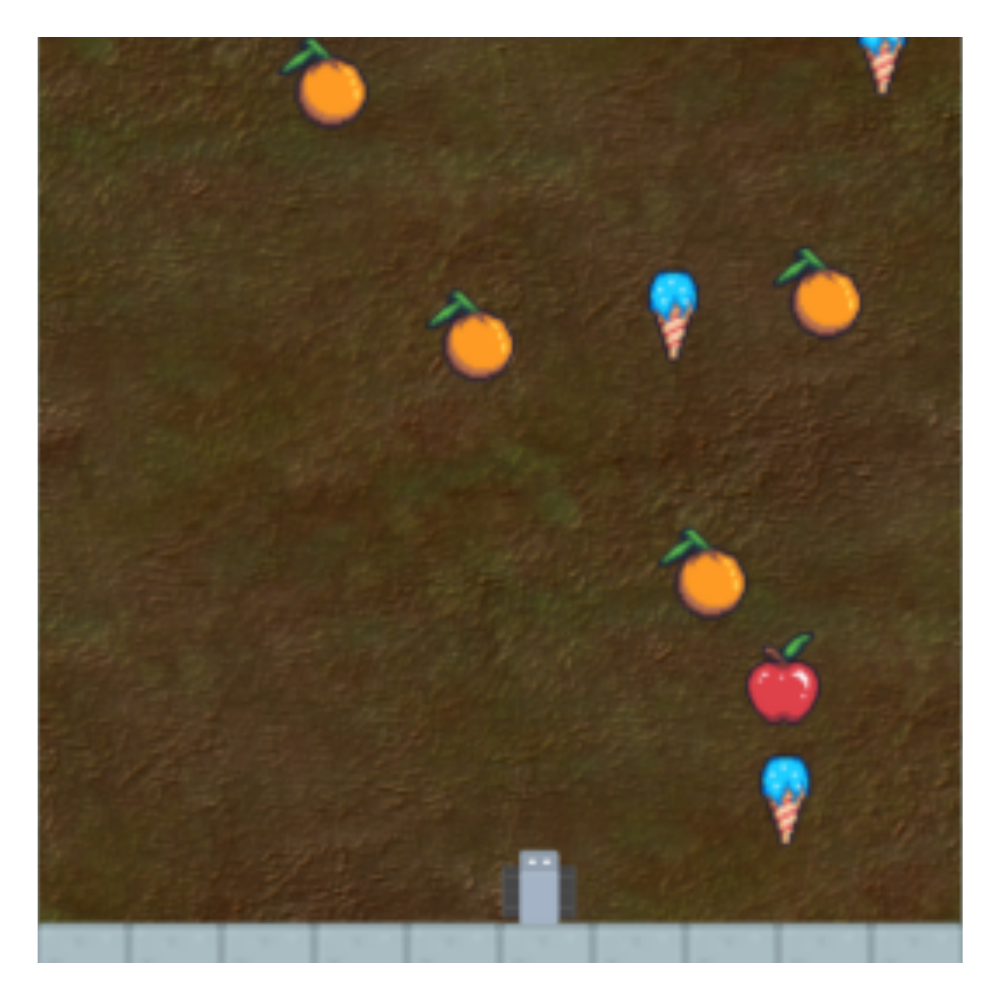}
    \caption{Fruit Bot}
  \end{subfigure}
  \caption{
    Environments studied in this paper. Environments (a-b) are multi-agent environments in which the ego-agent must learn online from other agents, and learn to navigate around other agents in the environment (see Section~\ref{sec:results_accelerate_rl}).
    Environments (c-d) are single-agent.
    In environment (c) we test whether the ego-agent can learn offline from a set of demonstrations previously collected by other agents (see Section~\ref{sec:results_transfer}. Environment (d) is used to test whether our method can scale to more complex, high-dimensional tasks (see Section~\ref{sec:results_accelerate_rl}).
  }
  \label{fig:envs}
  \vspace{-0.5em}
\end{figure*}

\textbf{Performance Bound.}
~Given a set of demonstration task vectors $\{w_k\}_{k=1}^{K}$ and successor features for the corresponding optimal policies $\{ \Psi^{\pi^{k}} \}_{k=1}^{K}$, \citet{barreto2017successor} show that it is possible to bound the performance of the GPI policy on the ego-agent task $w'$ as a function of the distance of $w'$ from the closest demonstration task $w_j$, and the value approximation error of the predicted value functions $\widetilde{Q}^{\pi_i} = \Psi^{\pi_i} w'$.
We extend this result to explicitly account for the reward approximation error $\delta_r$ obtained by the the learned cumulants  and the SF approximation error $\delta_\Psi$.
\begin{theorem}[Generalisation Bound of $\Psi \Phi$-Learning]
Let $\pi^*$ be the optimal policy for the ego task $w'$ and let $\pi$ be the GPI policy obtained from $\{\tilde{Q}^{\pi_i}\}$, with $\delta_r, \delta_\Psi$ the reward and successor feature approximation errors. Then $\forall s,a$
\begin{align}
     Q^{*}(s,a) - Q^\pi(s,a) \leq  &\frac{2}{1-\gamma} \bigg [ \phi_{\max} \min_{j} \| w' - w_j\| \hookleftarrow \nonumber \\ & + 2 \delta_r + \|w'\| \delta_{\Psi} + \frac{\delta_r}{1-\gamma} \bigg] \; .
\end{align}
\label{thm:gpi}
\end{theorem}
\begin{proof}
% \vspace{-2em}
See Appendix~\ref{app:proofs} for a formal statement.
% \vspace{-1em}
\end{proof}
In settings where the agent has a good reward and SFs approximation, and the ego task vector $w'$ is close to the demonstration tasks, Theorem~\ref{thm:gpi} says that the ego-agent will attain near-optimal performance from the start of training.

\section{Experiments}
\label{sec:experiments}

We conduct a series of experiments to determine how well $\Psi \Phi$-learning functions as an RL, IRL, imitation learning, and transfer learning algorithm.

\textbf{Baselines.}
~We benchmark against the following methods: (i) \textbf{DQN:} Deep Q-learning \cite{mnih2013playing}, (ii) \textbf{BC:} Behaviour Cloning, a simple imitation learning method in which we learn $p(\ba | \bs)$ via supervised learning on the demonstration data, (iii) \textbf{DQN+BC-AUX:} DQN with an additional behavior-cloning auxiliary loss \cite{hernandez2019agent}, (iv) \textbf{GAIL:} Generative Adversarial Imitation Learning \cite{ho2016generative}, which uses a GAN-like approach to approximate the expert policy, and (v) 
\textbf{SQILv2}: Soft Q Imitation Learning \cite{reddy2019sqil}, a recently proposed imitation technique that combines imitation and RL, and works in the absence of rewards.
For high-dimensional environments, we replace DQN with \textbf{PPO}, Proximal Policy Optimization \cite{schulman2017proximal}. Both DQN and PPO are trained to optimize environment reward through experience, and do not have access to other agents' experiences.

\subsection{Environments}
Experiments are conducted using four environments, shown in Figure \ref{fig:envs}. We cover a broad range of problem setting, including both multi-agent and single-agent environments, as well as learning online during RL training, or offline from previously collected demonstrations.

\textbf{Highway}~\citep{highway-env} is a multi-agent autonomous driving environment in which the ego-agent must safely navigate around other cars and reach its goal.
The other agents follow near-optimal scripted policies for various goals, depending on the scenario.
In the \textbf{single-task} scenario, other agents have the same objective as the ego-agent, so their experience is directly relevant.
In the \textbf{adversarial task}, the other agents do not move, and the ego-agent has to accelerate and go to a particular lane while avoiding other vehicles.
Finally, in the \textbf{multi-task} scenario, the other agents and ego-agent have different preferences over target speed, preferred lane, and following distance. We consider the multi-task scenario to be the most realistic and representative of real highway driving with human drivers. In addition to highway driving, we also study the more complex \textbf{Roundabout} task.
Roundabout is inherently multi-task, in that other agents randomly exit either the first or second exit, while the ego-agent must learn to take the third exit.

\textbf{CoinGrid} is a single-agent grid-world, environment containing goals of different colours. We collect offline trajectories of pre-trained agents with preferences for different goals. red.
During training, the ego-agent is only rewarded for collecting a subset of the possible goals. We can then test how well the ego-agent is able to transfer to a goal that was never experienced during training (Section \ref{sec:results_zeroshot}). This environment also enables learning easily interpretable preference vectors, allowing us to visualize how well our method works as an IRL method for inferring rewards (Section \ref{sec:results_irl}).

\textbf{FruitBot} is a high-dimensional, procedurally generated, single-agent environment from the OpenAI ProcGen~\citep{cobbe2020leveraging} suite.
We use FruitBot to test whether $\Psi \Phi$-learning can scale up to more complex RL environments, requiring larger deep neural network architectures that learn directly from pixels.
The agent must navigate around randomly generated obstacles while collecting fruit, and avoiding other objects and walls.
To create a multi-task version of FruitBot, we define additional tasks which vary agents' preferences over collecting objects in the environment, and train PPO baselines on these task variants.
The ego-agent observes the states and actions of these trained agents playing the game in parallel with its own interactions.

\begin{figure*}
  \centering
  \includegraphics[width=0.6\linewidth]{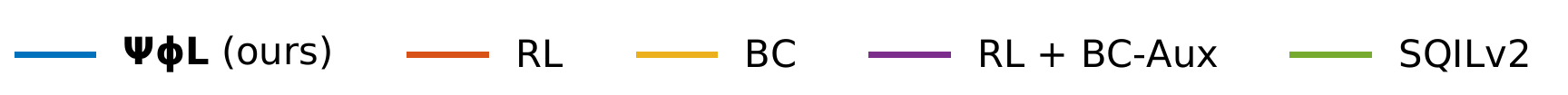}
  
  \begin{subfigure}[b]{0.24\linewidth}
    \centering
    \includegraphics[width=\linewidth]{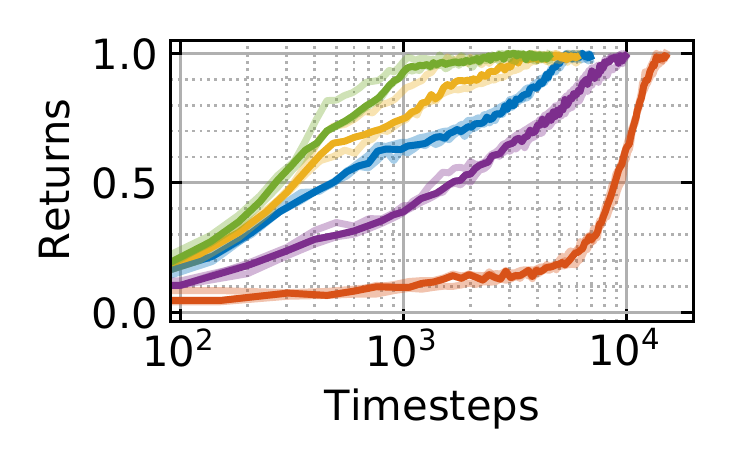}
    \caption{Highway: Single-task}
    \label{fig:results_curves_singletask}
  \end{subfigure}
  \begin{subfigure}[b]{0.24\linewidth}
    \centering
    \includegraphics[width=\linewidth]{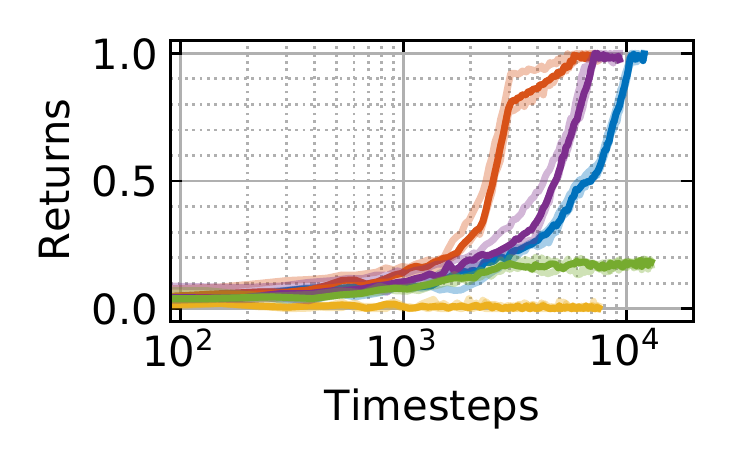}
    \caption{Highway: Adversarial}
    \label{fig:results_curves_adversarial}
  \end{subfigure}
  \begin{subfigure}[b]{0.24\linewidth}
    \centering
    \includegraphics[width=\linewidth]{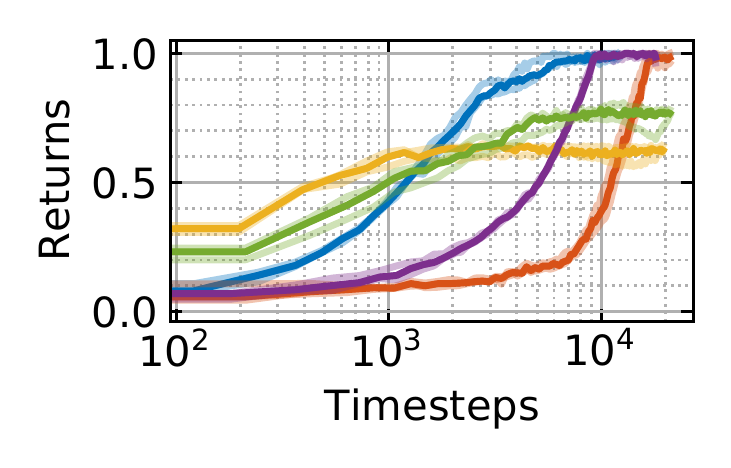}
    \caption{Highway: Multi-task}
    \label{fig:results_curves_multitask}
  \end{subfigure}
  \begin{subfigure}[b]{0.24\linewidth}
    \centering
    \includegraphics[width=\linewidth]{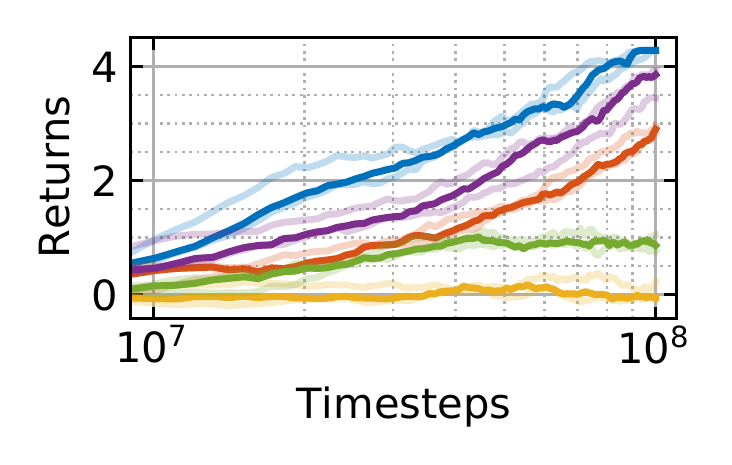}
    \caption{FruitBot}
    \label{fig:results_curves_fruitbot}
  \end{subfigure}

  \caption{
    Learning curves for $\Psi \Phi$-learning and baselines in three tasks in the multi-agent Highway environment (a-c), and in single-agent FruitBot (d).
    Tasks (a) and (b) represent extreme cases where either RL or imitation learning is irrelevant.
    In (a) other agents have the same task as the ego-agent, so imitation learning excels.
    In the adversarial task (b), other agents exhibit degenerative behaviour, so imitation learning performs extremely poorly and traditional RL (DQN) excels.
    In both of these extreme cases, $\Psi \Phi$-learning achieves good performance, showing it can flexibly reap the benefits of either imitation or RL as appropriate.
    Task (c) is most realistic; here, other agents have varied preferences and goals that may or may not relate to the ego-agent's task.
    $\Psi \Phi$-learning clearly outperforms baseline techniques. Similar results are shown in Fruitbot (d), showing that $\Psi \Phi$-learning scales well to high-dimensional environments, consistently outperforming baselines like PPO and SQIL.
    We plot mean performance over 3 runs and individual runs with alpha.
  }
  \label{fig:social_learning}
  \vspace{-1.5em}
\end{figure*}
\subsection{Accelerating RL with No-Reward Demonstrations}
\label{sec:results_accelerate_rl}
This section addresses two hypotheses:
\textbf{H1:} When the unlabelled demonstrations are relevant, $\Psi \Phi$-learning can accelerate or improve performance of the ego-agent when learning with online RL; and
\textbf{H2:} If the demonstrations are irrelevant, biased, or are generated by sub-optimal demonstrators, $\Psi \Phi$-learning can perform at least as well as standard RL. 

Figure \ref{fig:social_learning} shows the results of $\Psi \Phi$-learning and the baselines in the Highway and FruitBot environments. In the single-task scenario (\ref{fig:results_curves_singletask}), when other agents' experience is entirely relevant to the ego-agent's task, imitation learning methods like BC and SQILv2 learn fastest. DQN learns slowly because it does not use the other agents' experience. However, $\Psi \Phi$-learning achieves competitive results, out-performing DQN+BC-Aux \cite{hernandez2019agent,ndousse2020multi}.
In the adversarial task (\ref{fig:results_curves_adversarial}), the other agents' behaviors are irrelevant for the ego-agent's task, so imitation learning (BC and SQILv2) performs poorly, while traditional RL techniques (DQN and DQN+BC-Aux) perform best.
The performance of $\Psi \Phi$-learning does not suffer like other imitation learning methods; instead, it retains the performance of standard RL (\textbf{H2}).
$\Psi \Phi$-learning can flexibly reap the benefits of either imitation learning or RL, depending on what is most beneficial for the task.

The multi-task scenario (\ref{fig:results_curves_multitask}) is the most realistic autonomous driving task, in which other agents navigate the highway with varying driving styles.
Here, $\Psi \Phi$-learning clearly out-performs all other methods, suggesting it can leverage information about other agents' preferences in order to learn the underlying task structure of the environment, acclerating performance on the ego-agent's RL task (\textbf{H1)}. FruitBot (\ref{fig:results_curves_fruitbot}) gives consistent results, showing that $\Psi \Phi$-learning scales well to high-dimensional, single-agent tasks while still outperforming BC, SQIL, PPO, and PPO+BC-Aux.

\subsection{Inverse Reinforcement Learning}
\label{sec:results_irl}

We now test hypothesis \textbf{H3:} ITD is an effective IRL method, and can accurately infer other agents' rewards.
We present a quantitative and qualitative study of the rewards for other agents that are inferred by ITD, as well as the learned cumulants and preferences.
Here, we focus solely on offline IRL and use only ITD to learn from offline reward-free demonstrations, without any ego-agent experience.
\begin{table}
  \centering
  \caption{
    We evaluate how well $\Psi \Phi$-learning is able to infer the correct reward function by training an RL agent on the inferred rewards, and comparing this to alternative imitation learning methods in three environments.
    All methods are trained on expert demonstrations.
    A ``$\diamondsuit$'' indicates methods that infer an \emph{explicit} reward function and then use one of DQN or PPO to train an RL agent, depending on the environment.
    A ``$\clubsuit$'' indicates methods that directly learn a policy from demonstrations.
    A ``$\dagger$'' indicates methods that use privileged task id information for handling multi-task demonstrations.
    We report mean and standard error of \emph{normalised returns} over 3 runs, where higher-is-better and the performance is upper bounded by $1.0$, reached by the same RL agent, trained with the ground truth reward function.
    % \vspace{1em}
  }
  \label{tab:irl}
  \resizebox{\linewidth}{!}{
  \begin{tabular}{lc|c|c}
  \toprule
  \textbf{Methods} & \texttt{Roundabout}$^{\texttt{DQN}}$ & CoinGrid$^{\texttt{DQN}}$ & \texttt{FruitBot}$^{\texttt{PPO}}$ \\
  \midrule
  BC$^{\dagger \clubsuit}$~\citep{pomerleau1989alvinn}    &
    $0.81{\color{black!50}\pm0.02}$ &
    $0.69{\color{black!50}\pm0.06}$ &
    \bftab 0.37${\color{black!50}\pm0.02}$ \\
  SQIL$^{\dagger \clubsuit}$~\citep{reddy2019sqil}        &
    $0.85{\color{black!50}\pm0.02}$ &
    $0.64{\color{black!50}\pm0.05}$ &
    \bftab 0.35${\color{black!50}\pm0.03}$ \\
  \midrule
  GAIL$^{\dagger \diamondsuit}$~\citep{ho2016generative}     &
    $0.77{\color{black!50}\pm0.07}$ &
    $0.73{\color{black!50}\pm0.02}$ &
    $0.31{\color{black!50}\pm0.02}$ \\
  \rowcolor{ourmethod} ITD$^{\diamondsuit}$ (ours, cf. Section~\ref{subsec:inverse-temporal-difference-learning}) &
    \bftab 0.92${\color{black!50}\pm0.01}$ &
    \bftab 0.77${\color{black!50}\pm0.03}$ &
    \bftab 0.35${\color{black!50}\pm0.04}$\\
  \bottomrule
  \end{tabular}
  }
\vspace{-0.5em}
\end{table}

To quantitatively evaluate how well ITD can infer rewards, we train an RL agent on the inferred reward function, and compare the performance to other imitation learning and IRL methods. Table \ref{tab:irl} gives the performance in terms of normalised returns on all three environments.
Using ITD to infer rewards results in significantly higher performance than BC and SQIL, in two environments, and competitive performance in FruitBot. We note that unlike $\Psi \Phi$-learning, BC and SQIL directly learn a policy from demonstrations, and do not actually infer an explicit reward function. In contrast, GAIL does infer an explicit reward function, and ITD gives consistently higher performance than GAIL in all three environments. These results demonstrate that ITD is an effective IRL technique
(\textbf{H3}).

Qualitatively, we can evaluate how well the cumulants inferred by ITD in the CoinGrid environment span the space of possible goals.
We compute the learned cumulants $\hat{\phi}(s)$ for each square $s$ in the grid. 
Figure \ref{fig:cumulants_coingrid} shows the original CoinGrid game, and Figure \ref{fig:cumulants1}-\ref{fig:cumulants3} 
shows the first three dimensions of the learned cumulant vector, $\hat{\phi}_1$-$\hat{\phi}_3$ (the rest are given in the Appendix). We find that $\hat{\phi}_1$ is most active for red coins, $\hat{\phi}_2$ for green, and $\hat{\phi}_3$ for yellow.
Clearly, ITD has learned cumulant features that span the space of goals for this game.
See Appendix~\ref{app:visualisations} for more details and visualisations of the learned rewards and preferences.
\begin{figure}[!h]
% \vspace{-0.3cm}
  \centering
  \begin{subfigure}[b]{0.24\linewidth}
    \centering
    \includegraphics[width=\textwidth]{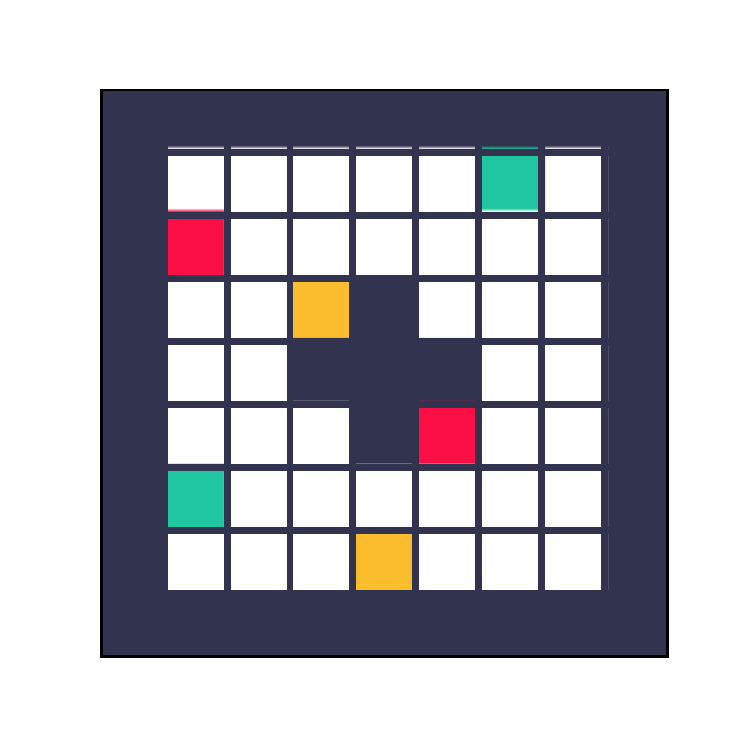}
    % \vspace*{-8mm}
    \caption{CoinGrid}
    \label{fig:cumulants_coingrid}
  \end{subfigure}%
  \begin{subfigure}[b]{0.24\linewidth}
    \centering
    \includegraphics[width=\textwidth]{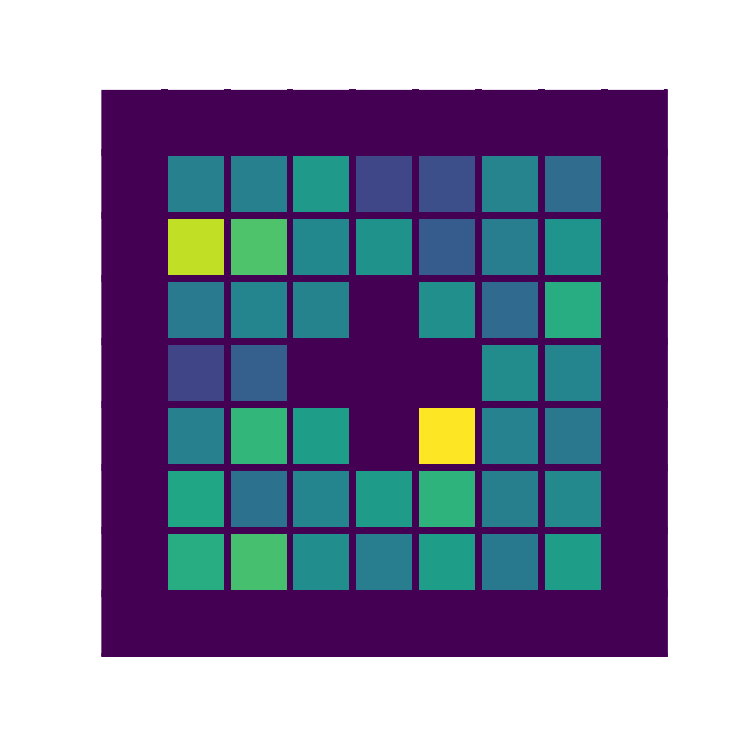}
    % \vspace*{-8mm}
    \caption{$\phi_{1}$}
    \label{fig:cumulants1}
  \end{subfigure}
  \begin{subfigure}[b]{0.24\linewidth} 
    \centering
    \includegraphics[width=\textwidth]{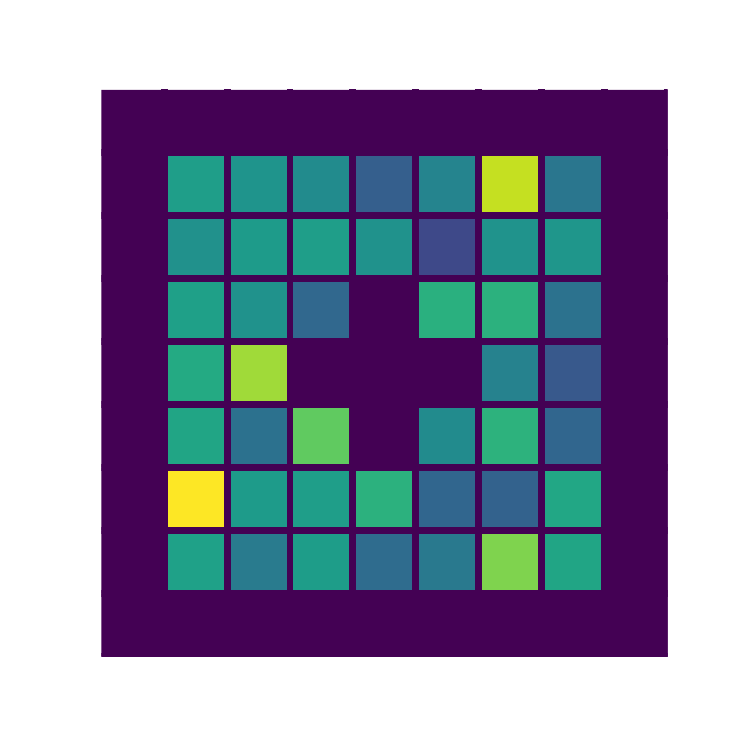}
    % \vspace*{-8mm}
    \caption{$\phi_{2}$}
  \end{subfigure}
  \begin{subfigure}[b]{0.24\linewidth}
    \centering
    \includegraphics[width=\textwidth]{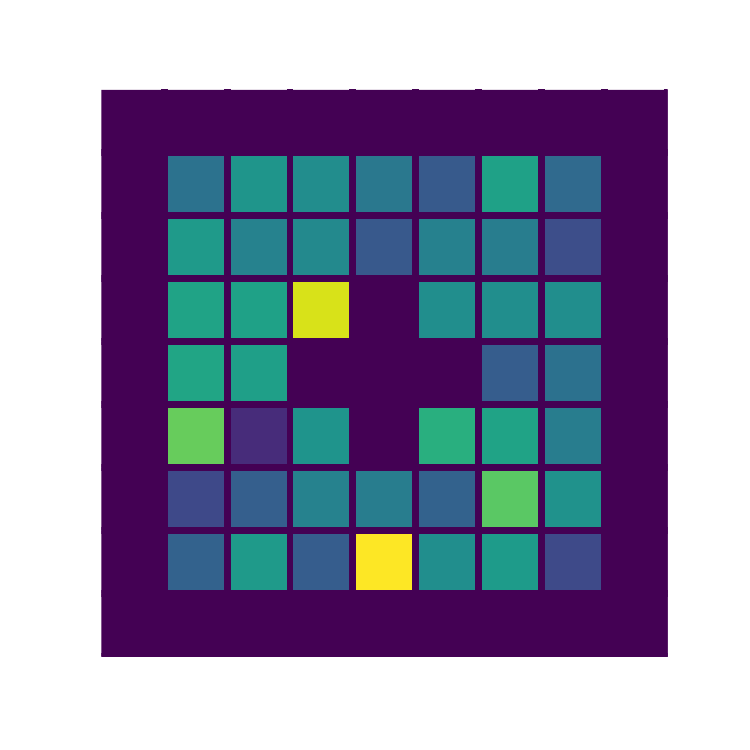}
    % \vspace*{-8mm}
    \caption{$\phi_{3}$}
    \label{fig:cumulants3}
  \end{subfigure}
  \label{fig:irl_fourrooms}
  \caption{
    Qualitative evaluation of the learned cumulants in the CoinGrid task.
    Cumulants $\phi_{1}$, $\phi_{2}$, and $\phi_{3}$ seem to capture the red, green, and yellow blocks, respectively.
    Therefore, linear combinations of the learned cumulants can represent arbitrary rewards in the environment, which involve stepping on the coloured blocks.
  }
  \label{fig:cumulants}
% \vspace{-1.5em}
\end{figure}

\subsection{Imitation learning}
\label{sec:results_imitate}
Here, we investigate hypothesis \textbf{H4:} $\Psi \Phi$-learning works as an effective imitation learning method, allowing for accurate prediction of other agents' actions. To test this hypothesis, we train other agents in the Roundabout environment, then split no-reward demonstrations from these agents into a train dataset (80\%) and a held-out test dataset. We use the train dataset to run ITD, which means that we use the data to learn both $\Phi$ and the $\Psi$ and $\bw$ for other agents. Because it is specifically designed to accurately predict other agents' actions, we use BC as the baseline. We compare this to using only ITD, and using the full $\Psi \Phi$-learning algorithm including ITD \textit{and} learning from RL and experience to update the shared cumulants $\Phi$. 

Accuracy in predicting other agents' actions on the held-out test set is used to measure imitation learning performance.
Figure \ref{fig:predict_others_actions} shows accuracy over the course of training.
At each phase change marked in the figure, the ego-agent is given a new task, to test how the representation learning benefits from diverse ego-experience.
We see that although BC obtains accurate train performance, it generalises poorly to the test set, reaching little over 80\% accuracy.
Without RL, $\Psi \Phi$-learning achieves similar performance. However, when using RL to improve imitation, $\Psi \Phi$-learning performs well on both the train and test set, achieving markedly higher accuracy ($\approx 95$\%) in predicting other agents' behaviour. This suggests that when $\Psi \Phi$-learning uses RL and interaction with the world to improve the estimation of the shared cumulants $\Phi$, this in turn improves its ability to model the $Q$ function of other agents and predict their behaviour. Further, $\Psi \Phi$-learning adapts well when the agent's goal changes, since it uses SFs to disentangle the representation of an agent's goal from environment dynamics. Taken together, these results demonstrate that $\Psi \Phi$-learning also works as a competitive imitation learning method (\textbf{H4}).

\begin{table*}[!h]
  \centering
  \caption{
    We evaluate how well $\Psi \Phi$-learning is able to transfer to new tasks in a few-shot fashion.
    We construct a multi-task variant of the CoinGrid environment: The ego-agent is provided demonstrations for either capturing only red coins \texttt{R} or only green coins \texttt{G}.
    Then it is evaluated on 4 different tasks: collecting (i) both red and green coins \texttt{R+G}, (ii) collecting red and avoiding green coins \texttt{R-G}, (iii) avoiding red and collecting green coins \texttt{-R+G} and (iv) avoiding both red and green coins \texttt{-R-G}.
    A ``$\diamondsuit$'' indicates methods that use a single model for all tasks, while ``$\clubsuit$'' indicates methods that require one model per task, i.e., they comprise of 4 models. Because it disentangles preferences from task representation, $\Psi \Phi$-learning is able to adapt to reach optimal performance on the new tasks after a single episode or improve intra-episode from the first episode after experiencing the first rewards. In contrast, SQIL takes $100$ episodes to adapt.
    % \vspace{1em}
  }
  \label{tab:few-shot}
  \resizebox{\linewidth}{!}{
  \begin{tabular}{lrrrr|rrrr|rrrr}
  \toprule
                                        &
  \multicolumn{4}{c}{0-shot}            &
  \multicolumn{4}{c}{1-shot}            &
  \multicolumn{4}{c}{100-shot}          \\
  \cline{2-5} \cline{6-9} \cline{10-13} \\
  \textbf{Methods}      &
  \texttt{R+G}          &
  \texttt{R-G}          &
  \texttt{-R+G}         &
  \texttt{-R-G}         &
  \texttt{R+G}          &
  \texttt{R-G}          &
  \texttt{-R+G}         &
  \texttt{-R-G}         &
  \texttt{R+G}          &
  \texttt{R-G}          &
  \texttt{-R+G}         &
  \texttt{-R-G}         \\
  \midrule
  SQILv2$^{\clubsuit}$~\citep{reddy2019sqil}                &
    $1.0{\color{black!50}\pm0.0}$                           &
    $0.0{\color{black!50}\pm0.0}$                           &
    $0.0{\color{black!50}\pm0.0}$                           &
    $-1.0{\color{black!50}\pm0.0}$                          &
    \bftab 1.0${\color{black!50}\pm0.0}$                    &
    $0.0{\color{black!50}\pm0.0}$                           &
    $0.0{\color{black!50}\pm0.0}$                           &
    $-1.0{\color{black!50}\pm0.0}$                          &
    \bftab 1.0${\color{black!50}\pm0.0}$                    &
    \bftab 1.0${\color{black!50}\pm0.0}$                    &
    \bftab 1.0${\color{black!50}\pm0.0}$                    &
    \bftab 1.0${\color{black!50}\pm0.0}$                    \\
  \rowcolor{ourmethod} $\Psi \Phi$-learning$^{\diamondsuit}$ (ours, cf. Section~\ref{subsec:psi-phi-learning-with-no-reward-demonstrations}) &
    $1.0{\color{black!50}\pm0.0}$                           &
    \bftab 0.2${\color{black!50}\pm0.1}$                    &
    \bftab 0.2${\color{black!50}\pm0.1}$                    &
    \bftab $-$0.4${\color{black!50}\pm0.2}$                   &
    \bftab 1.0${\color{black!50}\pm0.0}$                    &
    \bftab 1.0${\color{black!50}\pm0.0}$                    &
    \bftab 1.0${\color{black!50}\pm0.0}$                    &
    \bftab 1.0${\color{black!50}\pm0.0}$                    &
    \bftab 1.0${\color{black!50}\pm0.0}$                    &
    \bftab 1.0${\color{black!50}\pm0.0}$                    &
    \bftab 1.0${\color{black!50}\pm0.0}$                    &
    \bftab 1.0${\color{black!50}\pm0.0}$                    \\
  \bottomrule
  \end{tabular}
  }
% \vspace{-0.5em}
\end{table*}

\begin{figure}[h]
  \centering
  \begin{subfigure}[c]{0.63\linewidth}
    \centering
    \includegraphics[width=\textwidth]{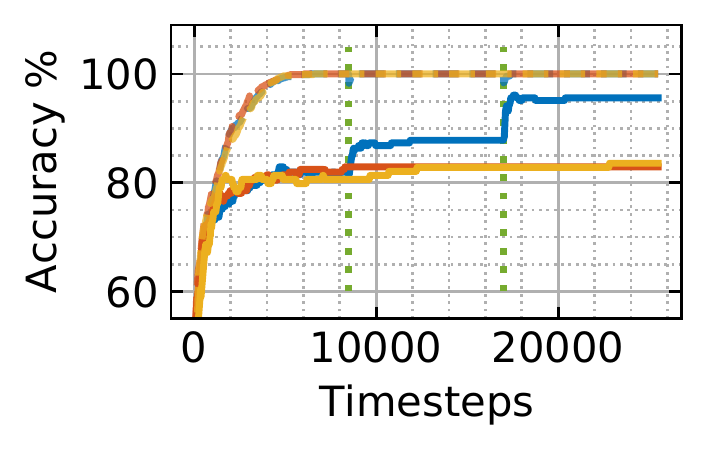}
  \end{subfigure}
  \begin{subfigure}[c]{0.35\linewidth}
    \centering
    \includegraphics[width=\textwidth]{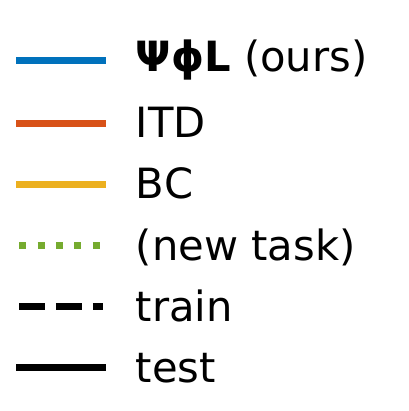}
  \end{subfigure}
  \caption{
    Test accuracy in predicting other agents' actions.
    The shared cumulants $\Phi$ for modelling others- and ego- reward functions allow our $\Psi \Phi$-learner to improve its ability to predict others' actions by experiencing new ego-tasks.
    Pure imitation learning and our ITD inverse RL methods achieve high train accuracy but they do not have a mechanism for utilising RL experience to improve their generalisation to the test set as new tasks are provided.
  }
  \label{fig:predict_others_actions}
%   \vspace{-0.cm}
\end{figure}

\subsection{Transfer and Few-Shot Generalisation}
\label{sec:results_transfer}

\label{sec:results_zeroshot}
Since SFs have been shown to improve generalization and transfer in RL, here we test hypothesis \textbf{H5:} $\Psi \Phi$-learning will be able to generalize effectively to new tasks in a few-shot transfer setting.
Using the CoinGrid environment, it is possible to precisely test whether the ego-agent can generalise to a task it has never experienced during training. Specifically, we would like to determine whether: (i) the ego-agent can generalise to tasks it was never rewarded for during training (but which it may have seen other agents demonstrate), and (ii) the ego-agent can generalise to tasks not experienced by \emph{any} agents during training.

Table \ref{tab:few-shot} shows the results of transfer experiments in which agents are given 0, 1, or 100 additional training episodes to adapt to a new task. Unlike SQIL, $\Psi \Phi$-learning is able to adapt 0-shot to obtain some reward on the new tasks, and fully adapt after a single episode to achieve the maximum reward on all transfer tasks. This is because $\Psi \Phi$-learning uses SFs to disentangle preferences (goals) in the task representation, and learn about the space of possible preferences from observing other agents.
To adapt to a new task, it need only infer the correct preference vector.
Task inference is trivially implemented as a least squares regression problem, see Eqn.~(\ref{eq:r-loss}): Having experienced $\B^{\text{new}}$ in the new task, the $\Psi \Phi$-learner identifies the preference vector for the new task by solving $\min_{\bw} \sum_{\B^{\text{new}}} \L_{\text{R}}(\btheta_{\Phi}, \bw)$.
In contrast, SQIL requires $100$ episodes to reach the same performance.

\section{Related Work}
\label{sec:related-work}

\textbf{Learning from demonstrations.}
~Learning from demonstrations, also referred to as imitation learning~\citep[IL,][]{widrow1964pattern,pomerleau1989alvinn,atkeson1997robot}, is an attractive framework for sequential decision making when reliable, expert demonstrations are available.
Early work on IL assumed access to high-quality demonstrations and aimed to match the expert policy~\citep{pomerleau1991efficient,heskes1998solving,ng2000algorithms,abbeel2004apprenticeship,billard2008survey,argall2009survey,ziebart2008maximum}.
Building on this assumption, many recent works have studied various aspects of both \emph{single-task}~\citep{ratliff2006maximum,wulfmeier2015maximum,choi2011inverse,finn2016guided,ho2016generative,finn2016connection,fu2017learning,zhang2018deep,rahmatizadeh2018vision} and \emph{multi-task}~\citep{,dimitrakakis2011bayesian,mulling2013learning,stulp2013learning,deisenroth2014multi,sharma2018multiple,codevilla2018end,fu2019language,rhinehart2020deep,filos2020can} IL.
However, these methods are all limited by the performance of the demonstrator that they try to imitate.
Learning from suboptimal demonstrations has been studied by~\citet{coates2008learning,grollman2011donut,zheng2014robust,choi2019robust,shiarlis2016inverse,brown2019extrapolating}, enabling, under certain assumptions, imitation learners to surpass their demonstrators' performance.
In contrast, our method integrates demonstrations into an online reinforcement learning pipeline and can use the demonstrations to improve learning on a new task.
Our inverse TD (ITD) learning offline multi-task inverse reinforcement learning algorithm is similar to the Cascaded Supervised IRL (CSI) approach \citep{klein2013cascaded}.
However, CSI assumes a single-task, deterministic expert while ITD does not.

\textbf{Reinforcement learning with demonstrations.}
~Demonstration trajectories have been used to accelerate the learning of RL agents~\citep{taylor2011integrating,vecerik2017leveraging,rajeswaran2017learning,hester2018deep,gao2018reinforcement,nair2018overcoming,paine2018one,paine2019making}, as well as demonstrations where actions and/or rewards are unknown~\citep{borsa2017observational,torabi2018behavioral,sermanet2018time,liu2018imitation,aytar2018playing,brown2019extrapolating}.
In contrast to the standard imitation learning setup, these methods allow improving over the expert performance as the policy can be further fine-tuned via reinforcement learning.
Offline reinforcement learning with online fine-tuning~\citep{kalashnikov2018scalable,levine2020offline} can be framed under this settings too.
Our method builds on the same principles, however, unlike these works, we do not assume that the demonstration data either come with reward annotations, or that they relate to the same task the RL agent is learning (i.e., we learn from multi-task demonstrations which may include irrelevant tasks).

\textbf{Successor features.}
~Successor features (SFs) are a generalisation of the successor representation~\citep{dayan1993improving} for continuous state and action spaces~\citep{barreto2017successor}. Prior work has used SFs for (i) zero-shot transfer~\citep{barreto2017successor,borsa2018universal,barreto2020fast}; (ii) exploration~\citep{janz2019successor,machado2020count}; (iii) skills discovery~\citep{machado2017eigenoption,hansen2019fast}; (iv) hierarchical RL~\citep{barreto2019option} and theory of mind~\citep{rabinowitz2018machine}.
Nonetheless, in all the aforementioned settings, direct access to the rewards or cumulants was provided.
Our method, instead, uses demonstrations without reward labels for inferring the cumulants and learning the corresponding SFs.
More closely to this work,~\citet{lee2019truly} propose learning cumulants and successor features for a \emph{single-task} IRL setting. 
Their approach differs from ours in two key respects:
first, they use a learned dynamics model to learn the cumulants. Second, the learned SFs are used for representing the action-value function, not to inform the behaviour policy with GPI.

\textbf{Model of others in multi-agent learning.}
~Our method draws inspiration and builds on the multi-agent learning setting, where multiple agents participate in the same environment and the states, actions of others are observed~\citep{davidson1999using,lockett2007evolving,he2016opponent,jaques2019social}.
However, we do not explore \emph{strategic} settings, where recursive reasoning~\citep{stahl1993evolution,yoshida2008game} is necessary for optimal behaviour.

% \vspace{-1em}
\section{Discussion}
\label{sec:discussion}

We have presented two major algorithmic contributions.
The first, ITD, is a novel and flexible offline IRL algorithm that discovers salient task-agnostic environment features in the form of cumulants, as well as learning successor features and preference vectors for each agent which provides demonstrations.
The second, $\Psi \Phi$-learning, combines ITD with RL from online experience. This makes efficient use of unlabelled demonstrations to accelerate RL, and comes with theoretical worst-case performance guarantees.
We showed empirically the advantages of these algorithms over various baselines: how imitation with ITD can improve RL and enable zero-shot transfer to new tasks, and how experience from online RL can help to improve imitation in turn.

\textbf{Future Work.}
~We want to explore ways to: (i) adapt $\Psi \Phi$-learning to multi-agent \emph{strategic} settings, where coordination and opponent modelling~\citep{albrecht2018autonomous} are essential and (ii) use a universal successor features approximator~\citep{borsa2018universal} for ITD, overcoming its current, linear scaling with the number of distinct demonstrators. 

\textbf{Acknowledgements.}
~We thank Pablo Samuel Castro, Anna Harutyunyan, RAIL, OATML and IRIS lab members for their helpful feedback.
We also thank the anonymous reviewers for useful comments during the review process.
A.F. is funded by a J.P.Morgan PhD Fellowship and C.L. is funded by an Open Phil AI Fellowship.

\bibliography{references}
\bibliographystyle{icml2021}

\appendix
\onecolumn

\section{Experimental Details}
\label{app:experimental-details}

In this section we describe the environments used in our experiments (see Section~\ref{sec:experiments}) and the experiment design.
\vspace{-1em}
\subsection{Highway}
\label{app:subsubsec:highway}

\begin{minipage}{0.82\textwidth}
  We build on the \texttt{highway-v0} task from the \texttt{highway-env} traffic simulator~\citep{highway-env}.
  The task is specified by:
  \begin{enumerate}[noitemsep]
    \item \textbf{State space, $\S$:}
      The kinematic information of the ego vehicle and the five closest vehicles (ordered from closest to the furthest) is used as the Markov state, i.e., $\bs_t = \{\left[ x_t, y_t, \dot{x}_{t}, \dot{y}_{t} \right]\}_{\text{ego},\ \text{other}_{1},\ \ldots,\ \text{other}_{5}} \in \mathbb{R}^{6 \times 4}$.
      The {\color{matlab-green}ego-car} is illustrated in green and the {\color{matlab-blue}other} cars in blue.
    \item \textbf{Action space, $\A$:}
      We use a discrete action space, constructed by $K$-means clustering of the continuous actions of the intelligent driving model~\citep{kesting2010enhanced}.
      We found out that keeping 9 actions was sufficient, i.e., $\ba_t \in \{0, \ldots, 8\}$.
    \item \textbf{Demonstrations, $\D$:}
      At each time-step, the ego-car observes online the state-action pairs for the 5 closest cars.
  \end{enumerate}
\end{minipage}
\hspace{0.25em}
\begin{minipage}{0.15\textwidth}
  \centering
  \includegraphics[width=\linewidth]{assets/misc/highway/highway-v0.png}
  \captionof{figure}{Highway}
\end{minipage}

\subsection{Roundabout}
\label{app:subsubsec:roundabout}

\begin{minipage}{0.82\textwidth}
  We build on the \texttt{roundabot-v0} task from the \texttt{highway-env} traffic simulator~\citep{highway-env}.
  The task is specified by:
  \begin{enumerate}[noitemsep]
    \item \textbf{State space, $\S$:}
      The kinematic information of the ego vehicle and the five closest vehicles (ordered from closest to the furthest) is used as the Markov state, i.e., $\bs_t = \{\left[ x_t, y_t, \dot{x}_{t}, \dot{y}_{t} \right]\}_{\text{ego},\ \text{other}_{1},\ \ldots,\ \text{other}_{3}} \in \mathbb{R}^{4 \times 4}$.
      The {\color{matlab-green}ego-car} is illustrated in green and the {\color{matlab-blue}other} cars in blue.
    \item \textbf{Action space, $\A$:}
      We use a discrete action space, constructed by $K$-means clustering of the continuous actions of the intelligent driving model~\citep{kesting2010enhanced}.
      We found out that keeping 6 actions was sufficient, i.e., $\ba_t \in \{0, \ldots, 5\}$.
    \item \textbf{Demonstrations, $\D$:}
      At each time-step, the ego-car observes online the state-action pairs for the 3 closest cars.
  \end{enumerate}
\end{minipage}
\hspace{0.25em}
\begin{minipage}{0.15\textwidth}
  \centering
  \includegraphics[width=\linewidth]{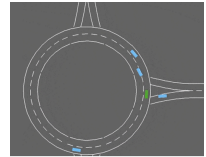}
  \captionof{figure}{Roundabout}
\end{minipage}

\subsection{CoinGrid}
\label{app:subsubsec:coingrid}

\begin{minipage}{0.82\textwidth}
  We build a simple multi-task grid-world.
  The task is specified by:
  \begin{enumerate}[noitemsep]
    \item \textbf{State space, $\S$:}
      We use a symbolic, multi-channel representation of the $7 \times 7$ gridworld~\citep{babyai_iclr19}:
      the first three channels specify the presence or absence of the three different coloured boxes, the forth channel was the walls mask and the fifth and last channel was the position and orientation of the agent.
      We represent the orientation of the agent by `painting' the cell in front of the agent.
      Therefore $\bs_t \in \{0, 1\}^{7 \times 7 \times 5}$.
    \item \textbf{Action space, $\A$:}
      We use the \{\texttt{LEFT, RIGHT, FORWARD}\} actions from Minigrid~\citep{gym_minigrid} to navigate the maze, i.e., $\ba_t \in \{0, 1, 2\}$.
    \item \textbf{Demonstrations, $\D$:}
      At the beginning of training, the agent is given state-action pairs of other agents collecting either red or green coins.
  \end{enumerate}
\end{minipage}
\hspace{0.25em}
\begin{minipage}{0.15\textwidth}
  \centering
  \includegraphics[width=\linewidth]{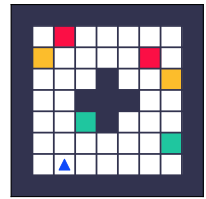}
  \captionof{figure}{CoinGrid}
\end{minipage}

\subsection{Fruitbot}
\label{app:subsubsec:fruitbot}

\begin{minipage}{0.82\textwidth}
  We build on the \texttt{Fruitbot} environment from OpenAI's ProcGen benchmark~\citep{cobbe2020leveraging}.
  The task is specified by:
  \begin{enumerate}[noitemsep]
    \item \textbf{State space, $\S$:}
      We use the original high-dimensional $64 \times 64$ RGB observations, i.e., $\bs_t \in \left[0, 1\right]^{64 \times 64 \times 3}$.
    \item \textbf{Action space, $\A$:}
      We use the original 15 discrete actions, i.e., $\ba_t \in \{0, \ldots, 14\}$.
    \item \textbf{Demonstrations, $\D$:}
      At each time-step, the agent observes online the states and actions of 3 trained agents playing the game in parallel:
      One agent collects both fruits and other objects, one collects other objects and avoids fruits and the last one randomly selects actions.
  \end{enumerate}
\end{minipage}
\hspace{0.25em}
\begin{minipage}{0.15\textwidth}
  \centering
  \includegraphics[width=\linewidth]{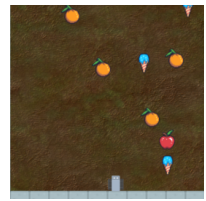}
  \captionof{figure}{Fruitbot}
\end{minipage}

\clearpage

\section{Implementation Details}
\label{app:implementation-details}

For our experiments we used Python~\citep{van1995python}.
We used JAX~\citep{jax2018github,deepmind2020jax} as the core computational library, Haiku~\citep{haiku2020github} and Acme~\citep{hoffman2020acme} for implementing $\Psi \Phi$-learning and the baselines, see Section~\ref{sec:experiments}.
We also used Matplotlib~\citep{hunter2007matplotlib} for the visualisations and Weightd \& Biases~\citep{wandb} for managing the experiments.

\subsection{Computation Graph}
\label{app:subsec:computational-graph}

\begin{figure}[h]
  \centering
\resizebox{0.8\linewidth}{!}{\begin{tikzpicture}

  %%%%%%%%%%%%
  %% SHARED %%
  %%%%%%%%%%%%

  {\color{matlab-red}
  \node (Phi) at (0,0) {$\Phi$};
  }
  
  %%%%%%%%%%%%
  %% OTHERS %%
  %%%%%%%%%%%%

  {\color{matlab-blue}
  \node[left = 3em of Phi] (Psi_k) {$\Psi^{k}$};
  \path (Phi) edge [>=latex, double, <->] node[sloped, anchor=center, below] {\tiny \ \ \ \ \ \ $\mathcal{L}_{\text{ITD}}$} (Psi_k);
  \node[left = 1em of Psi_k] (w_k) {$\bw^{k}$};
  \node[above = 1em of Psi_k] (Q_k) {$Q^{k}$};
  \path (Psi_k) edge [>=latex, ->] (Q_k);
  \path (w_k) edge [>=latex, ->, bend left=30] (Q_k);
  \node[draw, circle, fill=matlab-blue!5, above = 1.5em of Q_k] (D) {$\D$};
  \path (D) edge [>=latex, double, ->] node[sloped, right, rotate=90] {\tiny $\mathcal{L}_{\text{BC-}Q}$} (Q_k);
  \draw[thick, dotted] ($(w_k.south west)+(-0.3,0.0)$) rectangle ($(Q_k.north east)+(+0.3,+0.0)$);
  \node[above left = 0.10em and 0.00em of Q_k] (k) {\tiny $k = 1 \ldots K$};
  \node[above = 1em of D] (env_) {\includegraphics[width=1.5em]{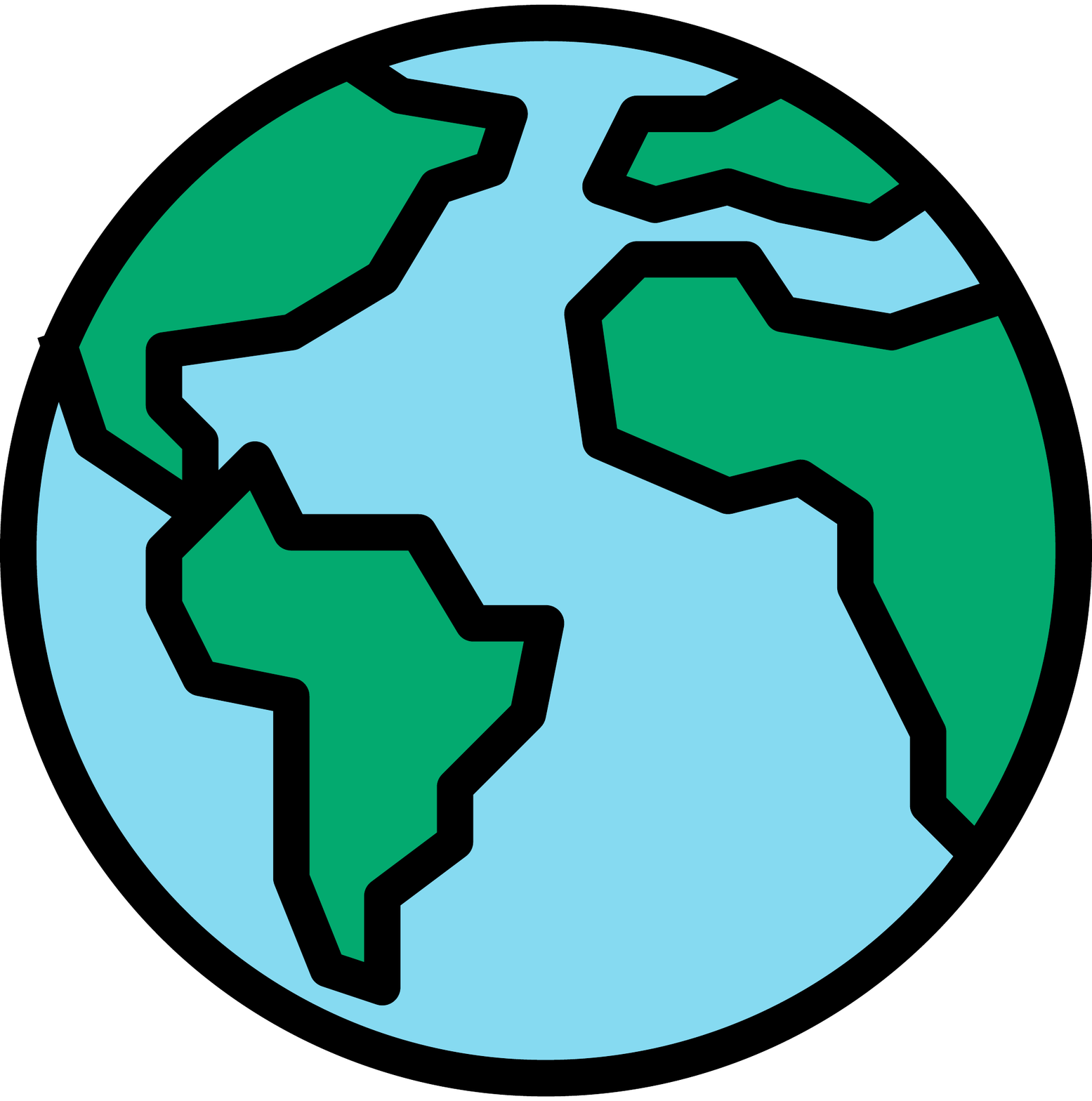}};
  \node[left = 1em of env_] (other_1) {\includegraphics[width=1.5em]{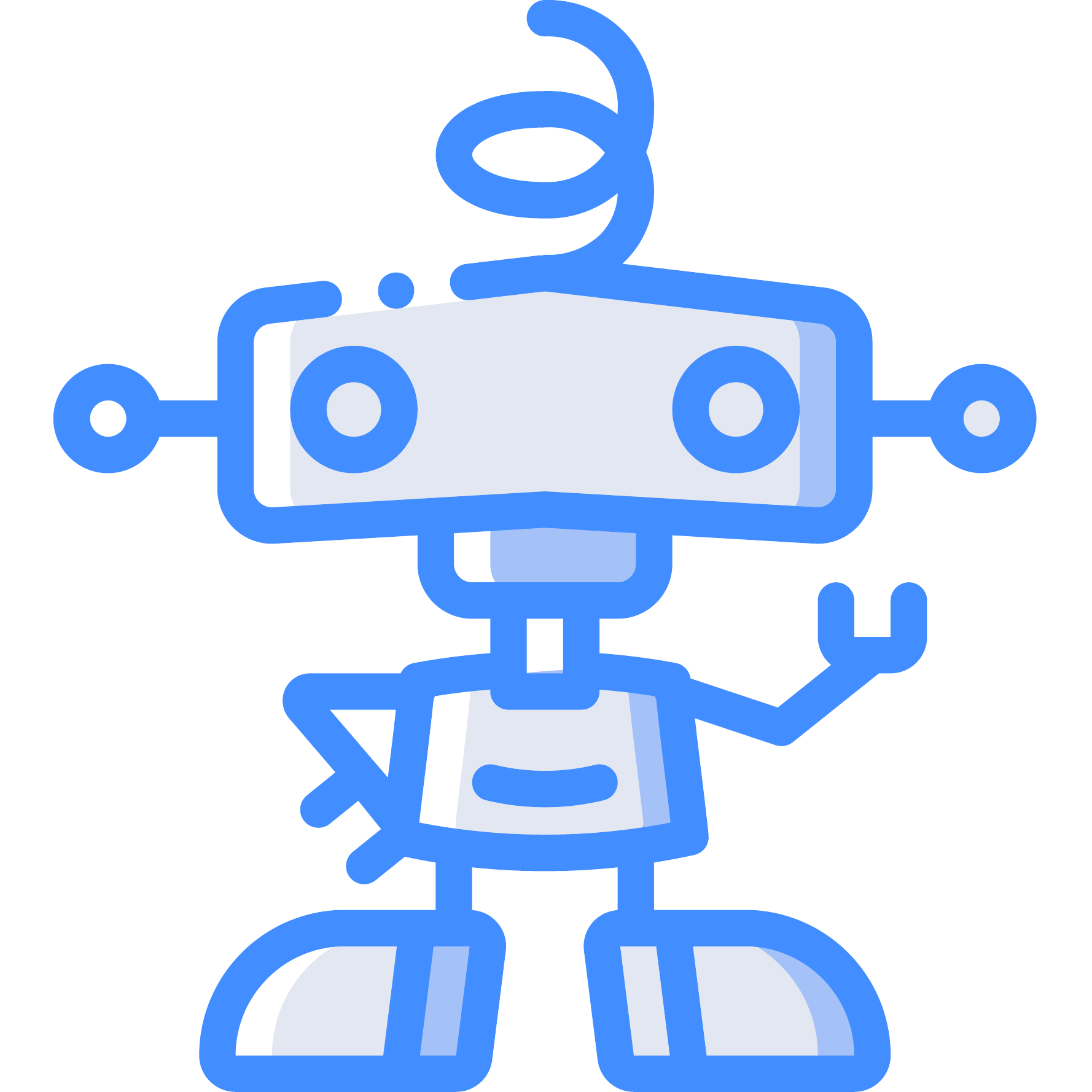}};
  \path (other_1) edge [>=latex, ->] (env_);
  \node[right = 1em of env_] (other_2) {\includegraphics[width=1.5em]{assets/misc/icons/blue-robot.pdf}};
  \path (other_2) edge [>=latex, ->] (env_);
  \path (env_) edge [>=latex, ->] (D);
  }
  
  %%%%%%%%%
  %% EGO %%
  %%%%%%%%%
  
  {\color{matlab-green}
  \node[above = 2em of Phi] (Psi_ego) {$\Psi^{\text{ego}}$};
  \path (Psi_ego) edge [>=latex, double, <-] node[sloped, right, rotate=90] {\tiny $\mathcal{L}_{\text{TD-}\Psi}$} (Phi);
  \node[right = 3em of Phi] (w_ego) {$\bw^{\text{ego}}$};
  \node[above = 2em of w_ego] (Q_ego) {$Q^{\text{ego}}$};
  \path (Psi_ego) edge [>=latex, ->] (Q_ego);
  \path (w_ego) edge [>=latex, ->] (Q_ego);
  
  \node[draw, above = 1.5em of Q_ego] (gpi) {\tiny GPI};
  \path (Q_ego) edge [>=latex, ->] (gpi);
  \path (Psi_k) edge [>=latex, dashed, ->, bend left=25] ([yshift=+0.25em]gpi.west);
  \path (w_ego) edge [>=latex, dashed, ->, bend left=60] (gpi);
  
  \node[above right = 0.5em and 1em of gpi] (pi_ego) {$\pi^{\text{ego}}$};
  \node[above = 0.1em of pi_ego] (ego_robot) {\includegraphics[width=1.25em]{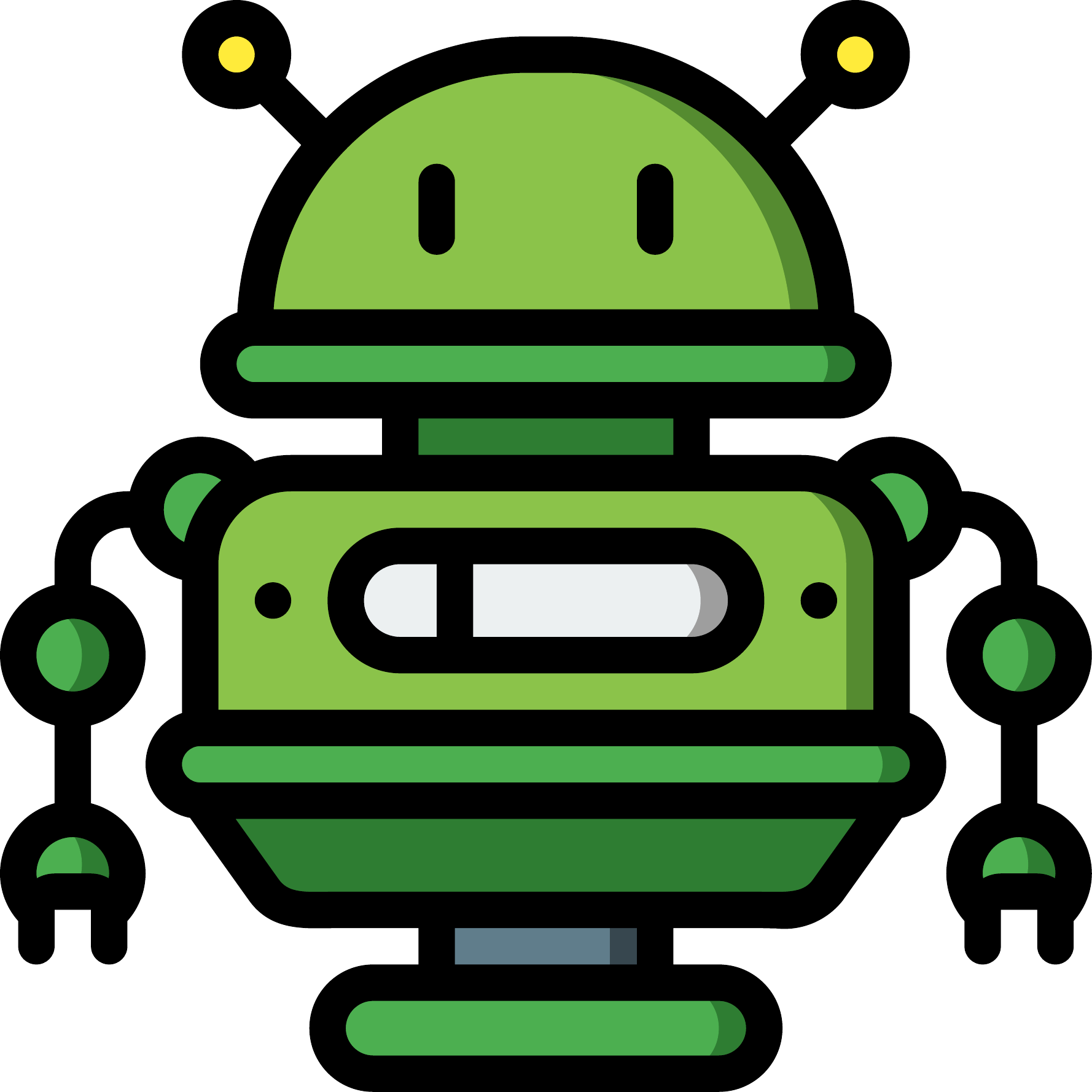}};
  \node[right = 1em of pi_ego] (env) {\includegraphics[width=1.5em]{assets/misc/icons/earth.pdf}};
  \path (gpi) edge [>=latex, ->, bend left=25] (pi_ego);
  \node[draw, circle, fill=matlab-green!5] (B) at (Q_ego -| env) {$\B$};
  \path (pi_ego) edge [>=latex, ->] (env);
  \path (env) edge [>=latex, ->] (B);
  \path (B) edge [>=latex, double, <->] node[sloped, anchor=center, above = 0.25em] {\tiny $\mathcal{L}_{\text{TD-}Q}$} (Q_ego);
  
  \node[below = 1em of w_ego] (r_ego) {$r^{\text{ego}}$};
  \path (Phi) edge [>=latex, ->, bend right=25] (r_ego);
  \path (w_ego) edge [>=latex, <->] (r_ego);
  \path (B) edge [>=latex, double, ->, bend left=45] node[sloped, anchor=center, below] {\tiny $\mathcal{L}_{R}$} (r_ego.east);
  }
    
\end{tikzpicture}}
  \caption{
    \textbf{Computational graph of the $\Psi \Phi$-learning algorithm.}
    Demonstrations $\D$ contain data from {\color{matlab-blue}other agents} for \emph{unknown} tasks. We employ \emph{inverse temporal difference learning} (ITD, see Section~\ref{subsec:inverse-temporal-difference-learning}) to recover other agents' successor features (SFs) and preferences.
    The {\color{matlab-green}ego-agent} combines the estimated SFs of others along with its own preferences and successor features with generalised policy improvement (GPI, see Section~\ref{subsec:successor-features-and-cumulants}), generating experience.
    Both the demonstrations and the ego-experience are used to learn the {\color{matlab-red}shared cumulants}.
    Losses $\mathcal{L}_{*}$ are represented with double arrows and gradients flow according to the pointed direction(s).
}
\label{fig:psi-phi-learning-schematic}
\end{figure}

\subsection{Neural Network Architecture}
\label{app:subsec:neural-network-architecture}

\begin{figure}[h]
  \centering
  \resizebox{0.75\linewidth}{!}{\begin{tikzpicture}
    
 %%%%%%%%%%%%
 %% SHARED %%
 %%%%%%%%%%%%
 
  \node (s_t) at (0,0) {$\bs_{t}$};

  \node [circle, inner sep=0pt, outer sep=0pt, minimum size=0pt] (dot_phi) at (0.0, 1.75) {};
  \node [circle, inner sep=0pt, outer sep=0pt, minimum size=0pt] (dot_ego) at (2.5, 1.75) {};
  \node [circle, inner sep=0pt, outer sep=0pt, minimum size=0pt] (dot_others_K) at (-2.5, 1.75) {};
  \node [circle, inner sep=0pt, outer sep=0pt, minimum size=0pt] (dot_others_k) at (-4.5, 1.75) {};
  \node [circle, inner sep=0pt, outer sep=0pt, minimum size=0pt] (dot_others_1) at (-6.5, 1.75) {};

  {\color{matlab-red}
  \draw [draw=matlab-red, fill=matlab-red!20] plot coordinates {(-0.5, 0.75) (+0.5, 0.75) (+0.25, 1.25) (-0.25, 1.25)} -- cycle;
  \node [circle, minimum size=2pt] (E) at (0.0, 1.0) {$\mathcal{E}$};
  \node[draw, minimum height=0.5cm, fill=matlab-red!10] (Phi) at (0.0, 2.5) {$\Phi(\bs_t, \cdot)$};
  
  \path (dot_phi.south) edge [>=latex, ->] (Phi);
  }

  \path (s_t) edge [>=latex, ->] (E);
  \path (E) edge [>=latex, -] (dot_phi.south);
  
  %%%%%%%%%%%%
  %% OTHERS %%
  %%%%%%%%%%%%

  {\color{matlab-blue}
  \node[draw, minimum height=0.5cm, fill=matlab-blue!10] (Psi_1_) at (-6.6, 2.6) {\tiny $\Psi^{1}(\bs_t, \cdot)$};
  \node[draw, minimum height=0.5cm, fill=matlab-blue!10] (Psi_1) at (-6.5, 2.5) {\tiny $\Psi^{1}(\bs_t, \cdot)$};
  \node[draw, circle, right = 0.1em of Psi_1, minimum height=0.5cm, fill=matlab-blue!10] (w_1) {\tiny $\bw^{1}$};
  \path (dot_others_1.south) edge [>=latex, ->] (Psi_1);
  \node[minimum height=0.5cm] (Psi_k) at (-4.5, 2.5) {\tiny $\cdots$};
  \node[right = 0.1em of Psi_k, minimum height=0.5cm] (w_k) {};
  \path (dot_others_k.south) edge [>=latex, ->] (Psi_k);
  \node[draw, minimum height=0.5cm, fill=matlab-blue!10] (Psi_2_) at (-2.6, 2.6) {\tiny $\Psi^{K}(\bs_t, \cdot)$};
  \node[draw, minimum height=0.5cm, fill=matlab-blue!10] (Psi_K) at (-2.5, 2.5) {\tiny $\Psi^{K}(\bs_t, \cdot)$};
  \node[draw, circle, right = 0.1em of Psi_K, minimum height=0.5cm, fill=matlab-blue!10] (w_K) {\tiny $\bw^{K}$};
  \path (dot_others_K.south) edge [>=latex, ->] (Psi_K);
  
  \path (dot_phi.south) edge [>=latex, -] (dot_others_K.south);
  \path (dot_others_K.south) edge [>=latex, -] (dot_others_k.south);
  \path (dot_others_k.south) edge [>=latex, -] (dot_others_1.south);
  }
  
  %%%%%%%%%
  %% EGO %%
  %%%%%%%%%
  
  {\color{matlab-green}
  \node[draw, minimum height=0.5cm, fill=matlab-green!10] (Psi_ego_) at (2.4, 2.6) {\tiny $\Psi^{\text{ego}}(\bs_t, \cdot)$};
  \node[draw, minimum height=0.5cm, fill=matlab-green!10] (Psi_ego) at (2.5, 2.5) {\tiny $\Psi^{\text{ego}}(\bs_t, \cdot)$};
  \node[draw, circle, right = 0.1em of Psi_ego, minimum height=0.5cm, fill=matlab-green!10] (w_ego) {\tiny $\bw^{\text{ego}}$};
  \path (dot_ego.south) edge [>=latex, ->] (Psi_ego);
  \path (dot_phi.south) edge [>=latex, -] (dot_ego.south);
  }
    
\end{tikzpicture}}
  \caption{
    \textbf{Neural network architecture of the $\Psi \Phi$-learner.}
    The rectangular nodes are tensors parametrised by MLPs and the circles are learnable vectors.
    We share an observation network/torso, $\mathcal{E}$, across all the network heads.
    The network heads that related to the {\color{matlab-blue}other agents} are in blue and trained from demonstrations $\D$.
    The {\color{matlab-green}ego-agent's} experience $\B$ is used for training the green heads.
    The {\color{matlab-red}shared cumulants and torso} are trained with both $\D$ and $\B$.
    An ensemble of two successor features approximators is used for the ego- and other- agents for combatting model overestimation, see Section~\ref{sec:method}.
}
\label{fig:neural-network-architecture}
\end{figure}
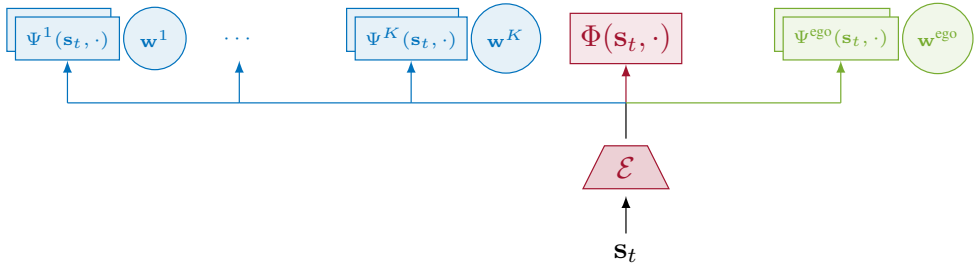

\subsection{Hyperparameters}
\label{app:subsec:hyperparameters}

\begin{table}[h]
  \centering
  \caption{
    \textbf{$\Psi \Phi$-learner's hyperparameters per environment}.
    The tuning was performed on a DQN~\citep{mnih2013playing} baseline with population based training~\citep{jaderberg2017population} using Weights \& Biases~\citep{wandb} integration with Ray Tune~\citep{liaw2018tune}.
    We selected the best hyperparameters configuration out of 32 trials per environment and used this for our $\Psi \Phi$-learner.
  }
  \label{tab:hyperparameters}
  \resizebox{\linewidth}{!}{
  \begin{tabular}{lc|c|c}
  \toprule
  & \texttt{Highway} & \texttt{CoinGrid} & \texttt{FruitBot} \\
  \midrule
  \textbf{Torso network, $\mathcal{E}$}  &
    MLP([512, 256]) &
    IMPALA~\citep{espeholt2018impala}, shallow (no LSTM)    &
    IMPALA~\citep{espeholt2018impala}, deep (no LSTM)       \\
  \textbf{Cumulants approximator, $\Phi$}   &
    MLP([128, 128])  &
    MLP([256, 128]) &
    MLP([256, 128]) \\
  \textbf{Successor features approximator, $\Psi$}   &
    MLP([256, 128])  &
    MLP([512, 256]) &
    MLP([512, 256]) \\
  \textbf{Ensemble size, $\Psi$}  &
    2   &
    2   &
    2   \\
  \textbf{$\L_{1}$ coefficient}  &
    0.05    &
    0.05    &
    0.05    \\
  \midrule
  \textbf{Number of dimensions in $\Phi$}   &
    8   &
    4   &
    64  \\
  \midrule
  \textbf{Minibatch size}    &
    512 &
    64  &
    32  \\
  $n$\textbf{-step} &
    4   &
    8   &
    128 \\
  \textbf{Discount factor, $\gamma$} &
    1.0     &
    0.9     &
    0.999   \\
  \textbf{Target network update period} &
    100     &
    1000    &
    2500    \\
  \textbf{Optimiser} &
    ADAM~\citep{kingma2014adam}, $\texttt{lr=1e-3}$ &
    ADAM~\citep{kingma2014adam}, $\texttt{lr=1e-4}$ &
    ADAM~\citep{kingma2014adam}, $\texttt{lr=5e-5}$ \\
  \bottomrule
  \end{tabular}
  }
% \vspace{-2em}
\end{table}

\subsection{Compute Resources}
\label{app:subsec:compute-resources}

All the experiments were run on Microsoft Azure \texttt{Standard\_NC6s\_v3} machines, i.e., with a 6-core vCPU, 112GB RAM and a single NVIDIA Tesla V100 GPU.
The iteration cycle for (i) \textbf{Highway} experiments was 3 hours; (ii) \textbf{CoinGrid} experiments was 5.5 hours and (iii) \textbf{Fruitbot} experiments was 19 hours.

\clearpage

\section{Proofs}
\label{app:proofs}
\setcounter{theorem}{0}
First, we formalise the statement of Theorem~\ref{thm:itd}.

\begin{theorem}[Validity of the ITD Minimiser]
% \vspace{-1em}
The minimisers of $\L_{\text{BC-}Q}$ and $\L_{\text{ITD}}$ are potentially-shaped cumulants that explain the observed reward-free demonstrations.
\end{theorem}

\begin{proof}
  First, we prove the validity of the minimiser of the inverse temporal difference learning for the single-task setting.
  Next, we show that the result holds true in the vector (i.e., cumulants) case.
  
  \textbf{Single task.}
  We assume that our demonstrations are generated by an expert, who samples actions from a Boltzmann policy, according to the optimal (for its task) action-value function $Q^{\pi_{\text{expert}}}$ and temperature $\nu>0$, i.e., $\D = \{(\bs, \ba)\} \sim \pi_{\text{expert}}$ s.t.
  \begin{align}
    \pi_{\text{expert}}(\ba | \bs) \triangleq p(A=\ba | \bs) = \frac{ \exp(\frac{1}{\nu} Q^{\pi_{\text{expert}}}(\bs, \ba))}{\sum_{a} \exp(\frac{1}{\nu} Q^{\pi_{\text{expert}}}(\bs, a))},\ \forall \bs \in \S, \ba \in \A \,.
  \end{align}
  The minimiser of the behavioural cloning loss $\L_{\text{BC-}Q}(\btheta_{Q})$, i.e. Eqn.~(\ref{eq:bcq-loss}), for a single expert is s.t.
  \begin{align}
    \btheta_{Q}^{*}
      &\in \argmin_{\btheta_{Q}}\ - \! \! \! \! \! \E_{(\bs, \ba) \sim \D} \log \frac{\exp(Q(\bs, \ba; \btheta_{Q}))}{\sum_{a} \exp(Q(\bs, a; \btheta_{Q}))} \\
      &\Rightarrow Q(\bs, \ba; \btheta_{Q}^{*}) = \frac{1}{\nu} Q^{\pi_{\text{expert}}}(\bs, \ba) + F(\bs),\ \forall \bs \in \S, \ba \in \A \,,
    \label{eq:btheta-q-star}
  \end{align}
  where $F: \mathcal{S} \rightarrow \mathbb{R}$ is a state-dependent (bounded potential) function.
  We arrive at Eqn.~(\ref{eq:btheta-q-star}) by (1) testing $\frac{1}{\nu} Q^{\pi_{\text{expert}}}(\bs, \ba)$ as a solution and noting that the ``softmax`` function is convex in the exponent~\citep{boyd2004convex} and (2) using the translation invariance property of the assumed Boltzmann policy parametrisation, i.e., for any $f: \mathcal{S} \times \mathcal{A} \rightarrow \mathbb{R}$ and $g: \mathcal{S} \rightarrow \mathbb{R}$
  \begin{align}
    \frac{\exp(f(\bs, \ba) + g(\bs))}{\sum_{a} \exp(f(\bs, a)  + g(\bs))} = \frac{\exp(g(\bs)) \exp(f(\bs, \ba))}{\exp(g(\bs)) \sum_{a} \exp(f(\bs, a))} = \frac{\exp(f(\bs, \ba))}{\sum_{a} \exp(f(\bs, a))} \,.
  \end{align}
  The minimiser of the inverse temporal difference learning loss $\L_{\text{ITD}}(\btheta_{Q}, \btheta_{r})$, Eqn.~(\ref{eq:itd-loss}), for a single expert is s.t.
  \begin{align}
    \btheta_{Q}^{*}, \btheta_{r}^{*} \in \argmin_{\btheta_{Q}, \btheta_{r}} \E_{(\bs, \ba, \bs', \ba') \sim \D} \| Q(\bs, \ba; \btheta_{Q}) - r(\bs, \ba; \btheta_{r}) - \gamma Q(\bs', \ba'; \btheta_{Q})  \|
    \label{eq:btheta-r-star}
  \end{align}
  where $\btheta_{Q}^{*}$ is minimising $\L_{\text{BC-}Q}(\btheta_{Q})$ simultaneously, as in Eqn.~(\ref{eq:btheta-q-star}).
  Therefore, it holds that $\L_{\text{ITD}}(\btheta_{Q}^{*}, \btheta_{r}^{*}) = 0$
  \begin{align}
    r(\bs, \ba; \btheta_{r}^{*})
      &= Q(\bs, \ba; \btheta_{Q}^{*}) - \gamma Q(\bs', \ba'; \btheta_{Q}^{*}) \\
      &\!\!\overset{(\ref{eq:btheta-q-star})}{=} \frac{1}{\nu} Q^{\pi_{\text{expert}}}(\bs, \ba) + F(\bs) - \gamma \frac{1}{\nu} Q^{\pi_{\text{expert}}}(\bs', \ba') - \gamma F(\bs') \\
      &= \, \, \frac{1}{\nu} \left[\underdescribe{Q^{\pi_{\text{expert}}}(\bs, \ba) - \gamma Q^{\pi_{\text{expert}}}(\bs', \ba')}{r^{\text{expert}}(\bs, \ba)}\right] + F(\bs) - \gamma F(\bs') \\
      &= \frac{1}{\nu} r^{\text{expert}}(\bs, \ba) + \underdescribe{F(\bs) - \gamma F(\bs')}{\substack{\text{potential-based reward} \\ \text{shaping function}}} \,,
      \label{eq:potential-shaping}
  \end{align}
  where $r^{\text{expert}}$ is the (unobserved) expert's reward function.
  We have shown that the minimiser of $\L_{\text{BC-}Q}$ and $\L_{\text{ITD}}$ leads to a reward function $r(\bs, \ba; \btheta_{r}^{*})$ which is a potential-based shaped and scaled reward function of the expert reward function and hence the optimal policy for $r(\bs, \ba; \btheta_{r}^{*})$ is also optimal for $r^{\text{expert}}(\bs, \ba)$ for all $\bs, \ba$~\citep{ng1999policy}.
  
  \textbf{Multiple tasks.}
  The minimiser of the behavioural cloning loss $\L_{\text{BC-}Q}(\btheta_{\Psi^{k}}, \bw^{k})$, i.e., Eqn.~(\ref{eq:bcq-loss}), for the $k$-th expert is s.t.
  \begin{align}
    \btheta_{\Psi^{k}}^{*}, \bw^{k*}
      &\in \argmin_{\btheta_{\Psi^{k}}, \bw^{k}}\ - \! \! \! \! \! \E_{(\bs, \ba) \sim \D^{k}} \log \frac{\exp(\Psi(\bs, \ba; \btheta_{\Psi^{k}})^{\top} \bw^{k})}{\sum_{a} \exp(\Psi(\bs, a; \btheta_{\Psi^{k}})^{\top} \bw^{k})} \\
      &\Rightarrow \Psi(\bs, \ba; \btheta_{\Psi^{k}})^{\top} \bw^{k} = \frac{1}{\nu} Q^{\pi_{k\text{-expert}}}(\bs, \ba) + F^{k}(\bs),\ \forall \bs \in \S, \ba \in \A \,,
    \label{eq:btheta-q-star-multi}
  \end{align}
  where $Q^{\pi_{k\text{-expert}}}$ is the $k$-agent's action-value function and $H^{k}: \mathcal{S} \rightarrow \mathbb{R}$ a state-dependent (bounded potential) function.
  Next, the minimiser of the inverse temporal difference learning loss $\L_{\text{ITD}}(\btheta_{\Psi^{k}}, \btheta_{\Phi})$, Eqn.~(\ref{eq:itd-loss}), for the $k$-th expert is s.t.
  \begin{align}
    \btheta_{\Psi^{k}}^{*}, \btheta_{\Phi}^{*} \in \argmin_{\btheta_{\Psi^{k}}, \btheta_{\Phi}} \E_{(\bs, \ba, \bs', \ba') \sim \D^{k}} \| \Psi(\bs, \ba; \btheta_{\Psi^{k}}) - \Phi(\bs, \ba; \btheta_{\Phi}) - \gamma \Psi(\bs', \ba'; \btheta_{\Psi^{k}})  \|
    \label{eq:btheta-r-star-multi}
  \end{align}
  where $\btheta_{\Psi^{k}}^{*}$ is minimising $\L_{\text{BC-}Q}(\btheta_{\Psi^{k}}, \bw^{k})$ simultaneously, as in Eqn.~(\ref{eq:btheta-q-star-multi}).
  Therefore, it holds that for $\L_{\text{ITD}}(\btheta_{\Psi^{k}}^{*}, \btheta_{\Phi}^{*}) = 0$
  \begin{align}
    \Phi(\bs, \ba; \btheta_{\Phi}^{*})
      &= \Psi(\bs, \ba; \btheta_{\Psi^{k}}^{*}) - \gamma \Psi(\bs', \ba'; \btheta_{\Psi^{k}}^{*}) \\
    \Phi(\bs, \ba; \btheta_{\Phi}^{*})^{\top} \bw^{k *}
      &= \Psi(\bs, \ba; \btheta_{\Psi^{k}}^{*})^{\top} \bw^{k *} - \gamma \Psi(\bs', \ba'; \btheta_{\Psi^{k}}^{*})^{\top} \bw^{k *} \\
      &\!\!\overset{(\ref{eq:btheta-q-star-multi})}{=} \frac{1}{\nu} Q^{\pi_{k\text{-expert}}}(\bs, \ba) + F^{k}(\bs) - \gamma \frac{1}{\nu} Q^{\pi_{k\text{-expert}}}(\bs', \ba') - \gamma F^{k}(\bs') \\
      &= \frac{1}{\nu} \left[\underdescribe{Q^{\pi_{k\text{-expert}}}(\bs, \ba) - \gamma Q^{\pi_{k\text{-expert}}}(\bs', \ba')}{r^{k\text{-expert}}(\bs, \ba)}\right] + F^{k}(\bs) - \gamma F^{k}(\bs') \\
      &= \frac{1}{\nu} r^{k\text{-expert}}(\bs, \ba) + \underdescribe{F^{k}(\bs) - \gamma F^{k}(\bs')}{\substack{\text{potential-based reward} \\ \text{shaping function}}} \,,
      \label{eq:potential-shaping-multi}
  \end{align}
  We have shown that the minimiser of $\L_{\text{BC-}Q}$ and $\L_{\text{ITD}}$ leads to agent-agnostic cumulants $\Phi(\bs, \ba; \btheta_{\Phi}^{*})$ and agent-specific preference vector $\bw^{k *}$, which when dot-producted, form a potential-based shaped and scaled reward function of the $k$-th expert reward function and hence the optimal policy for $\Phi(\bs, \ba; \btheta_{\Phi}^{*})^{\top} \bw^{k *}$ is also optimal for $r^{k\text{-expert}}(\bs, \ba)$ for all $\bs, \ba$~\citep{ng1999policy}.
  The result holds for all $k \in \{1 \ldots K\}$ since no assumptions were made for the proof about $k$.
\end{proof}

Next, we formalise the statement of Theorem 2.
When not specified the norm $\| \cdot \|$ refers to the 2-norm.
Given a function $F:\mathcal{X} \rightarrow \mathbb{R}^d$ for some finite set $\mathcal{X}$, we will write $F(x)$ to denote the value of the function on input $x$ and $F$ to denote the matrix representation of this function in $\mathbb{R}^{|\mathcal{X}| \times d}$.

\begin{theorem}[Generalisation Bound of $\Psi \Phi$-Learning]
Let $\mathcal{C} = (\mathcal{S}, \mathcal{A}, P, \gamma)$ be a CMP with a finite state space. Let $\phi: \mathcal{S} \rightarrow \mathbb{R}^d$, and let $\Phi = \phi(\mathcal{S}) \in \mathbb{R}^{|\mathcal{S}| \times d}$. Let $(r_i)_{i=1}^k$ denote a set of reward functions on $\mathcal{C}$, $\tilde{\Psi}^i$ be a collection of successor features approximations for policies $(\pi^i)_{i=1}^k$ ($\pi_i$ optimal for $r_i$) with true successor feature values $\Psi^i$, and $w_i$ the best least-squares linear approximator of $r_i$ given $\Phi$, with errors
\begin{equation*}
     \|\Phi w_i - r_i\|_\infty < \delta_r \text{ and } \| \tilde{\Psi}^i - \Psi^i \| < \delta_\Psi \quad \forall i.
\end{equation*}
Let $w'$ be a new preference vector for a reward function $r'$, with maximal error $\delta_r$ as well. Let $\tilde{Q}^i = \tilde{\Psi}^i w'$.  
Let $\pi^*$ be the optimal policy for the ego task $w'$ and let $\pi$ be the GPI policy obtained from $\{\tilde{Q}^{\pi_i}\}$, with $\delta_r, \delta_\Psi$ the reward and successor feature approximation errors. Then for all $s,a$
\begin{equation}
   Q^*(s,a) - Q^{\pi}(s,a)\leq  \frac{2}{1-\gamma} \bigg [(\phi_{\max}\|w_j - w'\|  + 2 \delta_r ) +  \|w'\|\delta_{\Psi} + \frac{1}{(1 - \gamma)} \delta_{r} \bigg ]
\end{equation}
\end{theorem}

\citet{barreto2017successor} construct their bound on the sub-optimality of the GPI policy as a function of the error of the value approximations $\widetilde{Q}^i$. Because we bound the reward approximation error, rather than the value approximation error, we require an additional step to obtain a bound on the errors of the value funciton approximations. To prove Theorem 1, we must therefore first use the following lemma to bound the effect of the \textit{reward approximation error} on the value approximation error. While this result is straightforward, we include a short proof for completeness.  
\begin{lemma}
Fix some policy $\pi$. Let $r$ be reward vector and let $w$ be the least-squares solution to $\min \| \Phi w - r\|$. Let $\Psi^\pi$ be the true successor features for $\Phi$ under policy $\pi$, and let $Q^\pi$ be the value. Let $\delta_r = R(\mathcal{S}) - \Phi w$, $\delta_{max} = \| \delta_r \|_\infty$. Then letting $\tilde{Q} = \Psi w$, we have
\begin{equation}
\| Q^\pi - \tilde{Q} \|_\infty \leq \frac{1}{1-\gamma} \delta_r
\end{equation}
\end{lemma}
\begin{proof}
\begin{align}
   \| Q^\pi - \tilde{Q} \|_\infty &\leq \sum \gamma^t \| P^{\pi t} (\Phi w - r) \|_\infty \\
 &\leq  \sum \gamma^t \|P^{\pi t} \delta_r \|_\infty = \sum_{t} \gamma^t \max_{s'}| \sum_{s \in \mathcal{S}} (P^{\pi})^t(s', s) \delta_r(s)| \\
    \intertext{Since $P^\pi$ is a stochastic matrix, so are all of its powers, and so the rows of $(P^\pi)^t$ sum to 1. }
 &\leq \sum_{t} \gamma^t \max_{s'} |\sum {P^{\pi}}^t(s, s')\delta_{\max}| = \sum \gamma^t \delta_{\max}\\ 
 &= \frac{1}{1-\gamma} \delta_{\max}
\end{align}
\vspace{-1em}
\end{proof}
We now prove the main result.

\begin{proof}
We follow the proof of~\citet[Theorem 2]{barreto2017successor}, with additional error terms to account for the reward and successor feature approximation errors.
\begin{align*}
    Q^*(s,a) - Q^{\pi}(s,a) &\leq Q^*(s,a) - Q^{\pi_j}(s,a) + \frac{2}{1-\gamma} \epsilon &\text{\citep[Theorem 1]{barreto2017successor}}\\ 
    &\leq \frac{2}{1-\gamma} \| r_j - r'\|_\infty +  \frac{2}{1 - \gamma} \epsilon  &\text{\citep[Lemma 1]{barreto2017successor}}\\
    &\leq \frac{2}{1-\gamma} \| \phi w_j + \delta_j - \phi w' - \delta'  \|_\infty +  \frac{2}{1 - \gamma} \epsilon \\
    &\leq \frac{2}{1-\gamma} (\phi_{\max}\|w_j - w'\|  + \delta_r + \delta_r ) + \frac{2}{1 - \gamma} \epsilon \\
    &\leq \frac{2}{1-\gamma} (\phi_{\max}\|w_j - w'\|  + 2 \delta_r ) + \frac{2}{1 - \gamma} \| \tilde{\Psi}^j w' - \Psi_j w' + \Psi_j w' - Q_j\| \\
    &\leq\frac{2}{1-\gamma}  (\phi_{\max}\|w_j - w'\|  + 2 \delta_r )  + \frac{2}{1 - \gamma} \| \tilde{\Psi}^j w' - \Psi_j w' \| + \|\Psi_j w' - Q_j\| \\
    &\leq \frac{2}{1-\gamma} (\phi_{\max}\|w_j - w'\|  + 2 \delta_r )  + \frac{2}{1-\gamma} \|w'\|\delta_{\Psi} + \frac{2}{1-\gamma} \| \Psi_j w' - Q_j\|  \\
    &\leq \frac{2}{1-\gamma}  (\phi_{\max}\|w_j - w'\|  + 2 \delta_r ) + \frac{2}{1-\gamma} \|w'\|\delta_{\Psi} + \frac{2}{1-\gamma} (\frac{1}{1-\gamma} \delta_{r} ) &\text{(Lemma 1)} \\
    &= \frac{2}{1-\gamma} \bigg [(\phi_{\max}\|w_j - w'\|  + 2 \delta_r ) +  \|w'\|\delta_{\Psi} + \frac{1}{(1 - \gamma)} \delta_{r} \bigg ]
\end{align*}
\end{proof}

\clearpage

\section{Algorithms}
\label{app:algorithms}

\begin{algorithm}
      \SetEndCharOfAlgoLine{}
      \SetKwComment{Comment}{$\triangleright$ }{}
      \SetKwInOut{Input}{Input}
      \SetKwInOut{Output}{Output}
      \Input{\\\hspace{-3.6em}
        \begin{tabular}[t]{l @{\hspace{1em}} l}
          $\D = \{(\bs_{1}, \ba_{1}, \ldots, \ba_{T}; k)_{k=1}^{K}\}$ & No-reward demonstrations \\
          $\lambda_{\bw}$ & $\L_{1}$ loss coefficient
        \end{tabular}
      }
      \Output{\\\hspace{-3.6em}
        \begin{tabular}[t]{l @{\hspace{1em}} l}
          $\btheta_{\Phi}$ & Parameters of cumulants network \\
          $\{\btheta_{\Psi^{k}}\}_{k=1}^{K}$ & Parameters of successor features approximators \\
          $\{\bw^{k}\}_{k=1}^{K}$ & Preferences vectors for the $K$ agents
        \end{tabular}
      }
      \BlankLine
      \tcp*[l]{initialisations}
      Initialise parameters $\btheta_{\Phi}, \{\btheta_{\Psi^{k}}, \bw^{k}\}_{k=1}^{K}$ \;
      \BlankLine
      \While{\texttt{budget}}{
        \BlankLine
        Sample trajectories $\{\tau_i = (s_{1}^{(i)}, a_{1}^{(i)}, \ldots, s_{T}^{(i)}, a_{T}^{(i)}; k^{(i)})\}_{i=1}^{N} \sim \D$ \;
        \BlankLine
        Calculate behavioural cloning loss $\L_{\text{BC-}Q}(\btheta_{\Psi^{k}}, \bw^{k})$ on samples $\{\tau_{i}\}_{i=1}^{N}$ \Comment*[r]{see Eqn.~(\ref{eq:bcq-loss})}
        $\btheta_{\Psi^{k}} \overset{\alpha}{\leftarrow} \nabla_{\btheta_{\Psi^{k}}} \L_{\text{BC-}Q}(\btheta_{\Psi^{k}}, \bw^{k})$ \Comment*[r]{update $\Psi$s}
        $\bw^{k} \overset{\alpha}{\leftarrow} \nabla_{\bw^{k}} \left(\L_{\text{BC-}Q}(\btheta_{\Psi^{k}}, \bw^{k}) + \lambda_{\bw} \| \bw^{k} \|_{1} \right)$ \Comment*[r]{update $\bw$s}
        \BlankLine
        Calculate inverse temporal difference loss $\L_{\text{ITD}}(\btheta_{\Phi}, \btheta_{\Psi^{k}})$ on samples $\{\tau_{i}\}_{i=1}^{N}$ \Comment*[r]{see Eqn.~(\ref{eq:itd-loss})}
        $\btheta_{\Phi} \overset{\alpha}{\leftarrow} \nabla_{\btheta_{\Phi}} \mathcal{L}_{\text{ITD}}(\btheta_{\Psi_{k}}, \btheta_{\Phi})$ \Comment*[r]{update $\Phi$}
        $\btheta_{\Psi_{k}} \overset{\alpha}{\leftarrow} \nabla_{\btheta_{\Psi_{k}}} \mathcal{L}_{\text{ITD}}(\btheta_{\Psi_{k}}, \btheta_{\Phi})$ \Comment*[r]{update $\Psi$s}
      }
      \caption{
        Inverse Temporal Difference Learning
      }
      \label{algo:itd}
\end{algorithm}

 \begin{algorithm}
      \SetEndCharOfAlgoLine{}
      \SetKwComment{Comment}{$\triangleright$ }{}
      \SetKwInOut{Input}{Input}
      \SetKwInOut{Output}{Output}
      \Input{\\\hspace{-3.6em}
        \begin{tabular}[t]{l @{\hspace{1em}} l}
          $\D = \{(\bs_{1}, \ba_{1}, \ldots, \ba_{T}; k)_{k=1}^{K}\}$ & No-reward demonstrations \\
          $\lambda_{\bw}$ & $\L_{1}$ loss coefficient
        \end{tabular}
      }
      \Output{\\\hspace{-3.6em}
        \begin{tabular}[t]{l @{\hspace{1em}} l}
          $\btheta_{\Psi^{\text{ego}}}$ & Ego successor features approximator \\
          $\btheta_{\Phi}$ & Parameters of cumulants network \\
          $\{\btheta_{\Psi^{k}}\}_{k=1}^{K}$ & Parameters of successor features approximators \\
          $\{\bw^{k}\}_{k=1}^{K}$ & Preferences vectors for the $K$ agents
        \end{tabular}
      }
      \BlankLine
      \tcp*[l]{initialisations}
      Empty replay buffer for ego-experience $\mathcal{B} = \{\}$ \;
      Initialise parameters $\btheta_{\Psi^{\text{ego}}}, \bw^{\text{ego}}, \btheta_{\Phi}, \{\btheta_{\Psi^{k}}, \bw^{k}\}_{k=1}^{K}$ \;
      \BlankLine
      \While{\texttt{budget}}{
        \BlankLine
        \tcp*[l]{agent-environment interaction}
        Reset episode, $\bs \leftarrow \texttt{env.reset()}$, $t \leftarrow 0$ \;
        \While{not \texttt{done}}{
          $\bw^{\text{ego}} \leftarrow \argmin_{w} \L_{\text{R}}(\btheta_{\Phi}, w; \B)$ \Comment*[r]{ego-task inference, see Eqn.~(\ref{eq:r-loss})}
          $\ba \leftarrow \pi^{\text{ego}}_{\text{GPI}} \left( \bs; \Psi^{\text{ego}}, \bw^{\text{ego}}, \{\btheta_{\Psi^{k}}\}_{k=1}^{K} \right)$ \Comment*[r]{GPI, see Eqn.~(\ref{eq:pi-ego})}
          Step in the environment, $\bs', r^{\text{ego}}, \texttt{done} \leftarrow \texttt{env.step(}\ba\texttt{)}$ \;
          Append transition in the replay buffer, $\mathcal{B} \leftarrow \mathcal{B} \cup \left( \bs, \ba, r^{\text{ego}}, \bs' \right)$ \;
          $\bs' \leftarrow \bs$, $t \leftarrow t + 1$ \;
          {\color{matlab-light-blue}(online demonstrations) Append demonstrations in $\D$ \Comment*[r]{optional}}
        }
        \BlankLine
        \tcp*[l]{parameter updates/learning}
        $\btheta_{\Phi}, \{\btheta_{\Psi^{k}}, \bw^{k}\}_{k=1}^{K} \leftarrow \texttt{ITD}\left( \D, \lambda_{\bw}, \btheta_{\Phi}, \{\btheta_{\Psi^{k}}, \bw^{k}\}_{k=1}^{K} \right)$ \Comment*[r]{see Algorithm.~(\ref{algo:itd})}
        \BlankLine
        Sample transitions $\{(\bs^{(i)}, \ba^{(i)}, r^{\text{ego}, (i)}, \bs'^{(i)})\} \sim \B$ \;
        \BlankLine
        Calculate the reward loss $\L_{\text{R}}(\btheta_{\Phi}, \bw^{\text{ego}})$ \Comment*[r]{see Eqn.~(\ref{eq:r-loss})}
        $\btheta_{\Phi} \overset{\alpha}{\leftarrow} \nabla_{\btheta_{\Phi}} \L_{\text{R}}(\btheta_{\Phi}, \bw^{\text{ego}})$ \Comment*[r]{update $\Phi$}
        \BlankLine
        Calculate TD losses $\mathcal{L}_{Q}(\btheta_{\Psi^{\text{ego}}})$ and $\mathcal{L}_{\text{TD-}\Psi}(\btheta_{\Psi^{\text{ego}}})$ \Comment*[r]{see Eqn.~(\ref{eq:q-loss},\ref{eq:psi-loss})}
        $\btheta_{\Psi^{\text{ego}}} \overset{\alpha}{\leftarrow} \nabla_{\btheta_{\Psi^{\text{ego}}}} \left( \mathcal{L}_{Q}(\btheta_{\Psi^{\text{ego}}}) + \frac{1}{| \Psi |}  \mathcal{L}_{\text{TD-}\Psi}(\btheta_{\Psi^{\text{ego}}}) \right)$ \Comment*[r]{update $\Psi^{\text{ego}}$}
      }
      \caption{
        $\Psi \Phi$-Learning
      }
      \label{algo:psiphi-learning}
\end{algorithm}

\clearpage

\section{Visualisations}
\label{app:visualisations}

\begin{figure}[h]
\vspace{-0.3cm}
  \centering
  \begin{subfigure}[b]{0.19\linewidth}
    \centering
    \includegraphics[width=\textwidth]{assets/misc/coingrid/board.pdf}
    \vspace*{-8mm}
    \caption{CoinGrid}
  \end{subfigure}%
  \begin{subfigure}[b]{0.19\linewidth}
    \centering
    \includegraphics[width=\textwidth]{assets/misc/coingrid/phi_1.pdf}
    \vspace*{-8mm}
    \caption{$\phi_{1}$}
  \end{subfigure}
  \begin{subfigure}[b]{0.19\linewidth} 
    \centering
    \includegraphics[width=\textwidth]{assets/misc/coingrid/phi_2.pdf}
    \vspace*{-8mm}
    \caption{$\phi_{2}$}
  \end{subfigure}
  \begin{subfigure}[b]{0.19\linewidth}
    \centering
    \includegraphics[width=\textwidth]{assets/misc/coingrid/phi_3.pdf}
    \vspace*{-8mm}
    \caption{$\phi_{3}$}
  \end{subfigure}
  \begin{subfigure}[b]{0.19\linewidth}
    \centering
    \includegraphics[width=\textwidth]{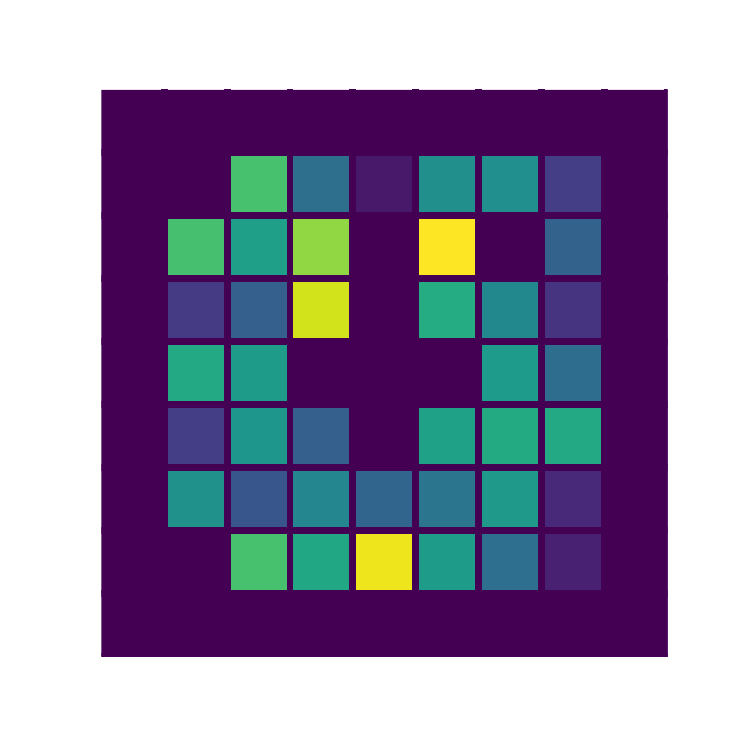}
    \vspace*{-8mm}
    \caption{$\phi_{4}$}
  \end{subfigure}
  \caption{
    Qualitative evaluation of the learned cumulants in the CoinGrid task.
    Cumulants $\phi_{1}$, $\phi_{2}$, and $\phi_{3}$ seem to capture the red, green, and yellow blocks, respectively.
    The yellow blocks are captured by both  and $\phi_{4}$.
    Therefore, linear combinations of the learned cumulants can represent arbitrary rewards in the environment, which involve stepping on the coloured blocks.
  }
  \label{fig:cumulants-complete}
\end{figure}

\begin{figure}[h]
\vspace{-0.3cm}
  \centering
  \begin{subfigure}[b]{0.3\linewidth}
    \centering
    \includegraphics[width=\textwidth]{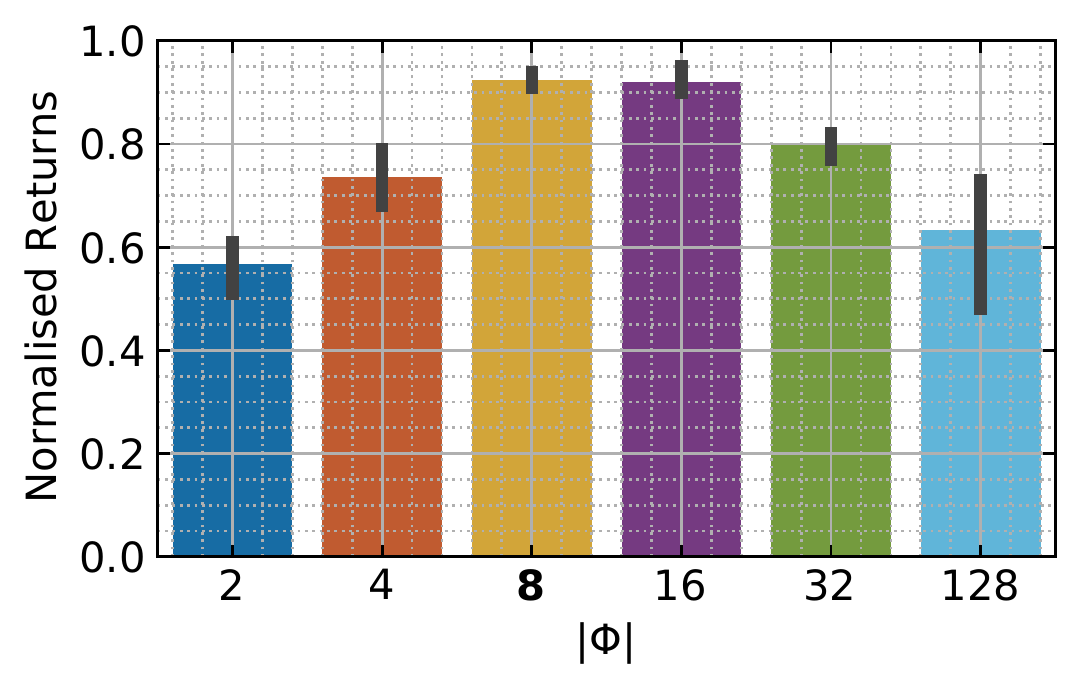}
    \caption{ITD for Roundabout}
  \end{subfigure}%
  \begin{subfigure}[b]{0.3\linewidth}
    \centering
    \includegraphics[width=\textwidth]{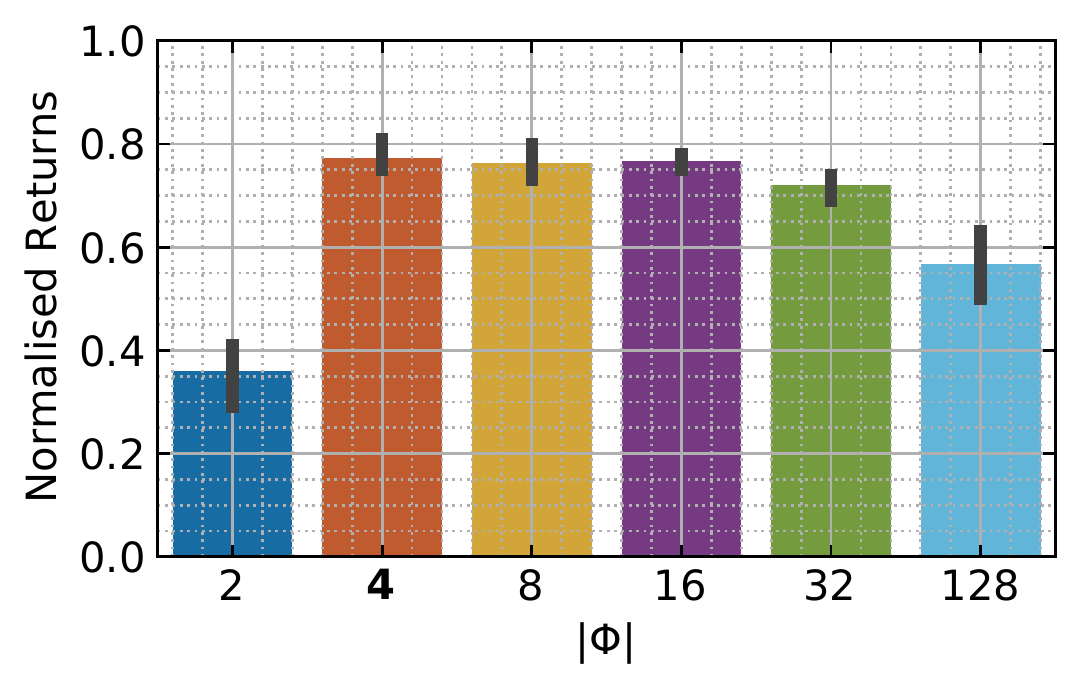}
    \caption{ITD for CoinGrid}
  \end{subfigure}
  \begin{subfigure}[b]{0.3\linewidth} 
    \centering
    \includegraphics[width=\textwidth]{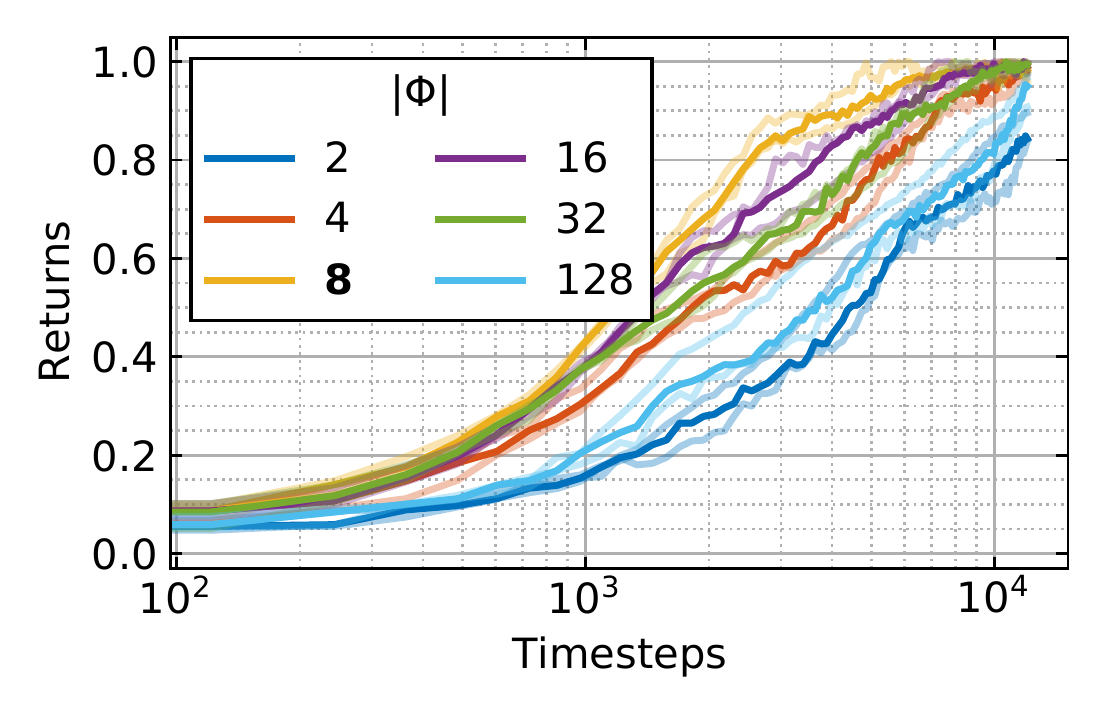}
    \caption{$\Psi \Phi$-learning for Highway Multi-Task} 
  \end{subfigure}
  \caption{
    Sensitivity of our ITD (see Section~\ref{subsec:inverse-temporal-difference-learning}) and $\Psi \Phi$-learning (see Section~\ref{subsec:psi-phi-learning-with-no-reward-demonstrations}) algorithms to the dimensionality of the learned cumulants.
    We consistently observe across all three experiments (a)-(c) that for a small number of $\Phi$ dimensions the cumulants are not expressive enough to capture the axis of variation of the different agents' reward functions (including the ego-agent in (c)).
    We also note that the performance of both ITD and $\Psi \Phi$-learning is relative robust for a medium and large number of $\Phi$ dimensions.
    We attribute this to the used sparsity prior, i.e., $\L_{1}$ loss, to the preferences $\bw$.
    In our experiments we selected the smallest number of $\Phi$ dimensions that demonstrated good performance to keep the number of model parameters as small as possible (in bold in the figures and reported in Table~\ref{tab:hyperparameters}).
  }
  \label{fig:sensitivity-to-size-of-cumulants}
\end{figure}

\end{document}